\documentclass[10pt]{article}

\usepackage[letterpaper,margin=1in]{geometry}
\usepackage[T1]{fontenc}
\usepackage[utf8]{inputenc}
\usepackage{amsmath,amssymb,amsthm}
\usepackage{booktabs}
\usepackage{array}
\usepackage{graphicx}
\usepackage{flafter}
\usepackage{float}
\usepackage{placeins}
\usepackage{microtype}
\usepackage{enumitem}
\usepackage{needspace}
\usepackage{subcaption}
\usepackage[hidelinks]{hyperref}

\graphicspath{{generated_current/}{generated_phase05/nonlinear_suffix/}{figures/}}
\setlist[itemize]{leftmargin=1.5em,itemsep=0.35em,topsep=0.4em}
\newcolumntype{L}[1]{>{\raggedright\arraybackslash}p{#1}}
\newtheorem{proposition}{Proposition}

\title{ObserverBench: Testing Mechanistic Estimates for Intervention and Control}
\author{Vijay Erramilli\\\texttt{evijay@gmail.com}}
\date{September 1, 2026}

\begin{document}
\maketitle

\begin{abstract}
Mechanistic interpretability is increasingly used to guide interventions, such
as steering activations, removing circuits, or enforcing safety boundaries.
But an internal estimate that is accurate on average can still lead to a poor
decision when it is used to choose a specific action.

We present \textbf{ObserverBench}, a benchmark framework for testing whether
an internal estimator---an observer---is adequate for the intervention,
control, or safety task it directs. Each task fixes the model, information
available to the observer, allowed actions, decision rule, unseen test cases,
and loss. ObserverBench reports both estimation accuracy and downstream action
performance. Its rankings are task-specific: observers are compared only
under the same information and action constraints.

Our theory and experiments show three ways these criteria can differ:

\begin{itemize}[leftmargin=1.4em,itemsep=0.15em,topsep=0.25em]
  \item \textbf{Task-relevant accuracy.} In a control loop, observer errors
  matter at the starting point and along directions the allowed intervention
  can reach.

  \item \textbf{Prediction versus decision.} On circuit-intervention tasks in
  GPT-2-small and Qwen2.5-7B, pairwise observers predict unseen intervention
  effects more accurately without always choosing better actions. Observers
  trained to predict action loss directly choose lower-loss actions. On one
  Qwen APPS panel, a 127-coefficient SAE readout has higher mean audit loss
  than its layer-matched 3,584-coordinate dense control, under a disclosed
  activation-density mismatch (observed $L_0=115$ versus reported $240$):
  compression without a better decision on that panel.

  \item \textbf{Consequence-aware safety.} Under fixed intervention budgets,
  a score that perfectly separates policy violations can allocate actions
  poorly when violations have different costs. On matched Qwen2.5-7B,
  Gemma-2-9B-it, and prospectively frozen Qwen3.5-9B APPS tasks, residual access
  leads on average, but the lowest-loss context changes---even between the two
  Qwen checkpoints---and AUROC can rank monitors differently from deployment
  loss. On Qwen3.5, an official Qwen-Scope SAE probe also trails its dense
  source residual under a declared base-to-posttraining checkpoint mismatch.
\end{itemize}

ObserverBench provides a common way to evaluate interpretability methods
through the actions they enable.

\end{abstract}

\section{Introduction}

Interpretability methods are increasingly used to steer and govern model
behavior rather than merely explain it. Internal estimates now guide concrete
decisions: which activation-steering direction to apply, which attention heads
to remove, whether to permit a risky tool call, or which requests receive a
limited human-review budget. Activation engineering and feedback-control
methods make these estimates part of an active control loop
\cite{turner2023activation,zou2023representation,nguyen2025feedback,skifstad2026local}.
Once an internal estimate is used to choose an action, statistical accuracy
alone is not enough
\cite{donti2017task,elmachtoub2022smart,orgad2026actionable}: an estimate can
perform well on average and still cause costly errors on individual prompts.

To analyze these systems, we separate the decision loop into three parts. The
\textbf{observer} reads measurements and estimates an internal state or
intervention effect. The \textbf{controller} uses that estimate to choose an
action. The \textbf{intervention} is the edit, ablation, or policy action it is
allowed to take. ObserverBench asks whether the observer is adequate for that
specific decision.

\Needspace{7\baselineskip}
Three failure modes illustrate why prediction accuracy and decision quality
must be reported separately. In safety triage, an observer can classify policy
violations perfectly yet spend a limited review budget on low-stakes cases
while allowing higher-impact harms through. This is \emph{consequence
disregard} (Sec.~\ref{sec:safety}). In model editing, an observer can predict a
circuit ablation's mean effect accurately yet choose an edit whose effect
varies widely across prompts. This is \emph{dispersion blindness}
(Sec.~\ref{sec:ioi}). An internal activation probe can also repeat information
already visible in the prompt or output logits. Such a \emph{redundant signal}
adds model access without improving the downstream decision
(Sec.~\ref{sec:safety}).

We introduce \textbf{ObserverBench}\footnote{Public workbench and submission
page: \url{https://kwisatzh.github.io/observerbench/}.},
a set of fixed, reusable tasks that makes these operational distinctions
measurable. Each task pins the model, information available to the observer,
allowed interventions, controller, measurement budget, held-out test cases,
and loss. Every observer faces the same controller and evaluation harness
within a task. The evaluator reports two outcomes separately: how accurately
the observer predicts its target and how much loss its chosen action incurs on
held-out cases.

ObserverBench focuses on the operational stage of interpretability.
Mechanistic tomography addresses the measurement problem: how to design probes
and interventions that recover hidden internal mechanisms and effect maps.
ObserverBench addresses the decision problem: whether a proposed estimate is
adequate to guide a particular intervention, control loop, or safety policy
\cite{erramilli2026tomography}. For safety tasks, ObserverBench freezes the audit policy and compares the
monitor inside it. AI-control evaluations ask the complementary, harder
question of whether the complete protocol remains safe when an adaptive model
tries to subvert it \cite{greenblatt2024aicontrol}. A monitor that performs
well here must still pass that end-to-end test.

An ObserverBench evaluation makes five choices explicit. It states the
\emph{target}: an internal state, an intervention's average effect, or perhaps
the loss of one action on one prompt. It states the \emph{information and
integrity boundary}: what the observer may read and who can alter, omit, delay,
or spoof those measurements. It fixes the \emph{action space} and the
\emph{decision rule and loss}, which determine how estimates become actions
and which errors matter. Finally, it declares the \emph{operating regime},
including prompts, intervention budgets, attack frequencies, episode horizons,
and attacker assumptions.

These choices make observer adequacy task-relative: permissions, tool
interfaces, budgets, attack rates, side channels, and evaluator-held evidence
are part of the evaluation, not background details. In the 2026 OpenAI--Hugging
Face incident, agents used an unsanctioned channel and spoofed some tool calls,
motivating the integrity boundary
\cite{openai2026securityincident,huggingface2026securityincident,
greenblatt2026hfinvestigation}. ObserverBench instantiates these choices in
three standardized operational modes---closed-loop control, effect prediction,
and safety triage---summarized in Sec.~\ref{sec:artifact}.

\Needspace{8\baselineskip}
\paragraph{Contributions.}

\begin{itemize}[itemsep=0.2em,topsep=0.3em]
  \item \textbf{An extensible evaluation workbench.} Versioned task and
  observer cards support local implementations and table-based submissions
  under one common evaluator (Sec.~\ref{sec:artifact}).

  \item \textbf{A formal adequacy test for control.} For linear readouts with
  an offset, an observer needs to match the target only at the starting point
  and along directions the allowed interventions can reach. An exact
  finite-horizon identity then separates residual reported error,
  starting-state mismatch, and target--observer response mismatch along the
  edit direction (Sec.~\ref{sec:design}, Eq.~\ref{eq:certificate}).

  \item \textbf{Prediction and compression are not decisions.} Pairwise
  models improve held-out effect prediction on GPT-2-small and Qwen2.5-7B, but
  better mean prediction does not reliably improve the chosen GPT-2 action.
  Relative to the same-feature mean observer, direct-loss training reduces
  action loss by 18.5\% in the clean GPT-2 confirmation and by 8.9\% in the
  Qwen aggregate \cite{wang2023ioi,olsson2022induction,qwen2024qwen25}. In the
  Qwen safety task, a sparse SAE readout uses far fewer coefficients yet incurs
  more escaped attacks: \emph{compression without a better decision}
  (Secs.~\ref{sec:phase7-confirmation}, \ref{sec:qwen-induction}, and
  \ref{sec:safety}) \cite{kantamneni2025saeprobes}.

  \item \textbf{Classification is not enough for safety.} In the authorization
  task, a label observer with AUROC 1.000 allocates a fixed intervention budget
  poorly when it ignores consequence. On the external Gemma APPS task, AUROC
  and realized violations reverse the order of two monitors. Across matched
  Qwen2.5-7B, Gemma-2-9B-it, and prospectively frozen Qwen3.5-9B tasks, residual
  observers lead on average, while the lowest-loss context changes---even
  between the two Qwen checkpoints
  (Sec.~\ref{sec:safety})
  \cite{greenblatt2024aicontrol,inglis2025controlarena,
  gemmateam2024gemma2,lieberum2024gemmascope}.
\end{itemize}

The next question is operational: what exactly does ObserverBench fix, what
may a researcher submit, and what does the evaluator return?

\section{How ObserverBench evaluates an observer}
\label{sec:artifact}

A fair comparison keeps the decision problem fixed. Otherwise, lower loss
could reflect a better observer, a more capable controller, easier actions, or
a different test distribution. ObserverBench therefore treats these choices
as part of the task.

Each evaluation uses a \textbf{TaskCard}, an \textbf{ObserverCard}, and a
\textbf{result scorecard} in four steps:

\begin{enumerate}[leftmargin=1.4em,itemsep=0.2em,topsep=0.3em]
  \item \textbf{Choose a task.} The TaskCard defines the decision problem using
  the five choices above and fixes the training measurements, held-out cases,
  budgets, and clean-performance checks. These do not change when observers
  are compared.
  \item \textbf{Submit an observer.} Any method may be used if it respects the
  task's information and integrity boundary. The ObserverCard records what the
  observer reads, where those measurements come from, how it was fit or
  calibrated, what model access it requires, and any known failure conditions.
  For a control task, it also declares whether the submission replaces the
  state estimate, the edit direction, or both. If both change, the score applies
  to the pair.
  \item \textbf{Run the common evaluator.} Every observer receives the same
  held-out cases and faces the same controller, actions, budgets, and loss.
  \item \textbf{Compare outcomes and cost.} The result scorecard reports prediction
  quality, downstream action loss, and the access or measurement cost of
  obtaining the estimate.
\end{enumerate}

There are two ways to cross the observer--evaluator boundary. A local
implementation can fit and run a method that needs model inference or internal
activations. A public submission can instead provide a table of predictions or
scores for the task's released cases. The second route lets a researcher test
an observer without installing the original model or experiment stack. In both
routes, the task---not the observer---applies the held-out targets, controller,
and loss.

The three operational modes share this contract but ask different questions
(Table~\ref{tab:operational-modes}).

\begin{table}[H]
\centering
\small
\caption{ObserverBench operational modes. Each row defines a complete,
task-relative comparison.}
\label{tab:operational-modes}
\begin{tabular}{@{}L{0.14\linewidth}L{0.23\linewidth}L{0.31\linewidth}L{0.23\linewidth}@{}}
\toprule
\textbf{Mode} & \textbf{Observer supplies} & \textbf{Evaluator fixes} &
\textbf{Main outputs} \\
\midrule
Closed-loop control & A state estimate, an edit direction, or a declared pair
& Model, starting states, allowed edits, controller, horizon, and loss
& Estimation error, trajectory error, and collateral movement \\
\addlinespace
Effect prediction & Predicted effects or losses for held-out interventions
& Calibration measurements, held-out interventions, action rule, and loss
& Prediction error and selected-action loss \\
\addlinespace
Safety triage & One risk or consequence score per request or trajectory
& Test cases, prevalence, action policy, capacity budget, and consequence loss
& Caught and missed harm, false interventions, and clean utility \\
\bottomrule
\end{tabular}
\end{table}

ObserverBench does not collapse unlike modes or task versions into a single
overall rank. Each scorecard identifies which observer is better for one
declared use.

The workflow makes results comparable. Three additional rules keep each
comparison meaningful:

\begin{itemize}[leftmargin=1.5em,itemsep=0.2em,topsep=0.3em]
  \item \textbf{Separate construction from evaluation.} Fix targets, splits,
  and interventions before measuring held-out performance. Check first that
  the unedited model succeeds on the base task and required groups; we call
  this the \emph{clean stop gate}.

  \item \textbf{Make the comparison observable.} The measurement design must
  actually exercise the effects being compared. For example, an interaction
  cannot be assessed if the relevant component combination is almost never
  measured. Repeated control should report error across the trajectory rather
  than only a threshold crossing.

  \item \textbf{Evaluate the quantity used to choose the action.} Mean-effect
  prediction, prompt-level intervention loss, and consequence-weighted safety
  risk are different targets. The task should score the one the controller
  actually uses.
\end{itemize}

This contract tells us how to compare observers, but it leaves a prior
question unanswered: what must an observer preserve for its estimate to be
adequate for control? Section~\ref{sec:design} answers that question by
restricting accuracy to the states that the allowed interventions can reach.

\section{Closed-loop control: what an observer must preserve}
\label{sec:design}

Once the controlled model, initial states, allowed intervention rule,
controller, and loss are fixed, only observer errors that affect reachable
states can change the controlled trajectory. An observer turns a model
measurement into information that can guide an intervention. The distinction
between the estimate and the intervention is important. An observer may
correctly report how far the model is from a target while pointing the
controller in an unhelpful direction; conversely, a useful direction can be
paired with a poor estimate.

For model state $h$, let the observer estimate a quantity $\widehat z$ and let
$d$ denote the proposed intervention direction:
\[
E(h)=\widehat z,
\qquad
D(h)=d.
\]
The controller uses the estimate to choose an action magnitude,
\[
u=C(y^\star,\widehat z).
\]
Within this fixed system, testing $E$ keeps $D$ fixed, testing $D$ keeps $E$
fixed, and changing both identifies only the combined system.

\Needspace{6\baselineskip}
An observer need not be correct on every possible model state. It needs to be
correct on the states that the allowed interventions can reach. To state this
concretely, suppose an activation has three coordinates but the controller can
move only the first two. A wrong coefficient on the third coordinate cannot
affect this controller. A wrong coefficient on the first can.

To state the same idea precisely, consider an affine---linear plus an
offset---target and observer:
\[
z(h)=b_T+\beta_T^\top h,
\qquad
\widehat z(h)=b_E+\beta_E^\top h.
\]
Let $h_0$ be the starting state. Let the columns of $B$ span the allowed edit
directions. From $h_0$, the controller can reach only
\[
\mathcal H_R=h_0+\mathcal U,
\qquad
\mathcal U=\operatorname{span}(B).
\]
The observer only has to match the target on this reachable set. With one fixed
direction $d$, this means that its error along $d$ must be zero.

\begin{proposition}[Reachable-subspace adequacy]
\label{prop:reachable}
Define
\[
\delta_0=z(h_0)-\widehat z(h_0),
\qquad
\delta_\beta=\beta_T-\beta_E.
\]
Then $z(h)=\widehat z(h)$ for every $h\in\mathcal H_R$ if and only if
\[
\delta_0=0
\qquad\text{and}\qquad
P_{\mathcal U}\delta_\beta=0,
\]
where $P_{\mathcal U}$ is the orthogonal projector onto $\mathcal U$.
\end{proposition}

\begin{proof}
Every reachable state has the form $h=h_0+Bv$. The readout difference is
$\delta_0+\delta_\beta^\top Bv$. It is zero for every $v$ exactly when
$\delta_0=0$ and $B^\top\delta_\beta=0$, which is equivalent to
$P_{\mathcal U}\delta_\beta=0$.
\end{proof}

With one fixed direction $d$, Proposition~\ref{prop:reachable} reduces adequacy
to $(\beta_T-\beta_E)^\top d=0$.

\subsection{Why a convergent observer can still miss the target}

The proposition tells us where an observer must agree with the target. The
next question is what happens when its remaining error enters a feedback loop.
Consider an unclipped controller that repeatedly edits along one direction:
\[
h_{t+1}=h_t+K\widehat e_t d,
\qquad
\widehat e_t=y^\star-\widehat z(h_t).
\]
Here $y^\star$ is the desired value, $K$ sets the edit size, and
$\widehat e_t$ is the error reported by the observer. Three quantities govern
the loop:
\[
g_E=\beta_E^\top d,
\qquad
g_T=\beta_T^\top d,
\qquad
p=1-Kg_E.
\]
The response $g_E$ is what the observer expects along the edit direction, while
$g_T$ is the target's actual response. The convergence factor $p$---called the
closed-loop pole in control theory---says how quickly the error reported by the
observer shrinks. After $T$ steps, the total action is
\[
A_T=
\begin{cases}
\widehat e_0(1-p^T)/g_E, & g_E\neq0,\\
TK\widehat e_0, & g_E=0,
\end{cases}
\]
and the two errors are exactly
\[
\widehat e_T=p^T\widehat e_0,
\qquad
e_T=e_0-g_T A_T,
\]
where $e_t=y^\star-z(h_t)$. For $g_E\neq0$, define
\[
m_0=z(h_0)-\widehat z(h_0),
\qquad
\rho=\frac{g_T-g_E}{g_E}.
\]
Substitution gives the finite-horizon certificate
\begin{equation}
\label{eq:certificate}
e_T=p^T\widehat e_0-m_0-\rho(1-p^T)\widehat e_0.
\end{equation}

\Needspace{8\baselineskip}
The equation separates three sources of error:

\begin{itemize}
  \item $p^T\widehat e_0$ is the error the observer still reports after $T$
  steps;
  \item $m_0$ is a mismatch at the starting state; and
  \item $\rho(1-p^T)\widehat e_0$ is the mismatch between the response the
  observer expects and the response the target actually has.
\end{itemize}

The observer's reported error converges when $|p|<1$, or $0<Kg_E<2$ for
positive $K$, but this controls only the first term. For example, if $p^T$ is
nearly zero and $m_0=0.2$, the observer reports success while the true target
remains off by about $0.2$, even before response mismatch is counted. A
convergent observer is not necessarily tracking the true target.

\subsection{A simple control certificate remains predictive in a Transformer}
\label{sec:nonlinear-suffix}

Equation~\ref{eq:certificate} is exact for an affine system. A Transformer is
nonlinear. Our first experiment asks whether the same three sources of
error---remaining reported error, starting mismatch, and response mismatch---
remain useful when every edit passes through learned Transformer layers. We
begin with a small model whose target we can verify exactly, following the
controlled ``wind-tunnel'' approach of \cite{agarwal2025bayesian}.

\paragraph{Task contract.}
This experiment instantiates the control mode as follows:

\begin{table}[H]
\centering
\small
\begin{tabular}{@{}L{0.21\textwidth}L{0.71\textwidth}@{}}
\toprule
Dimension & Fixed contract \\
\midrule
Target & Scalar state readout and its error over a 15-step trajectory. \\
Information boundary & The residual after the first Transformer block;
first-order observers use the two individual features, while lifted observers
also use their product. \\
Action space & Repeated residual edits along the declared observer direction. \\
Decision rule and loss & Proportional feedback with $K=0.35$; integrated
squared tracking error is primary, with final error and clipping as diagnostics. \\
Operating regime & Twelve independently trained two-block Transformers for
additive and interactional targets, with edits checked against clean residual
scale. \\
\bottomrule
\end{tabular}
\caption{ObserverBench contract for the closed-loop control experiment.}
\label{tab:control-task-contract}
\end{table}

We train a two-block Transformer on sequences with two binary features, $x_1$
and $x_2$, and target
\[
y=0.35x_1+0.25x_2+\gamma x_1x_2+\epsilon.
\]
The product $x_1x_2$ matters only when both features are present. A first-order
observer reads $x_1$ and $x_2$ separately, so it cannot represent this joint
condition. A lifted observer adds the pairwise feature $x_1x_2$. We edit the
residual after block 0 and rerun block 1, the final layer normalization, and the
target head. The edit therefore passes through the learned nonlinear suffix;
we do not replace that suffix with a linear approximation.

We run two tests with 12 independently trained models per target family:

\begin{itemize}
  \item \textbf{Does the certificate predict failure?} We fix the observer
  response at $g_E=0.85$ and vary the response mismatch $\rho$, initial bias,
  and target offset. Small positive and negative edits measure the actual target
  response $g_T$. We then ask whether Equation~\ref{eq:certificate} predicts
  which rollouts miss the target and by how much.

  \item \textbf{Does the interaction coordinate improve control?} We compare
  the natural first-order and pairwise estimator--direction pairs. We use
  $\gamma=1.15$ when the target contains the interaction and $\gamma=0$ as the
  additive control.
\end{itemize}

Integrated squared tracking error counts error across the full trajectory, not
only at the final step, and uncertainty is measured across the 12 independently
trained models, not across the many conditions within each model. None of the
main comparisons clips. Every natural comparison and 99.2\% of controlled
conditions move no farther than the largest distance among the four clean task
archetypes. The edits therefore have a familiar scale, although four archetypes
cannot prove that edited states remain on the full activation manifold.

The suffix is measurably nonlinear. At the largest offsets, its mean deviation
from the affine prediction is 16.1\% of output movement, and the certificate
residual is 20.2\% of mean absolute final error. We therefore ask whether the
affine quantities remain predictive, not whether the Transformer is linear.

They remain predictive. The main checks agree:

\begin{itemize}
  \item The raw certificate correlates $0.952$ with actual final error and gets
  the sign right in 96.0\% of conditions.
  \item Its residual MAE is 16.5\% of mean absolute actual error.
  \item After fitting only an intercept and slope on the other 11 seeds, its
  leave-one-seed-out MAE is $0.063$, compared with $0.302$ for the bias-only
  model and $0.367$ for the observer-pole-only model. It wins on every held-out seed.
\end{itemize}

Across this checked operating range, the three certificate terms capture most
of the error. This does not show that the suffix or its activation manifold is
locally linear.

\begin{table}[t]
\centering
\scriptsize
\begin{tabular}{@{}lrrrr@{}}
\toprule
Comparison & A & B & A--B [95\% CI] & Seeds \\
\midrule
Bias vs. full cert. & 0.302 & 0.063 & 0.238 [0.230, 0.246] & 12/12 \\
Pole vs. full cert. & 0.367 & 0.063 & 0.303 [0.294, 0.312] & 12/12 \\
FO vs. lifted, $\gamma=1.15$ & 2.632 & 1.870 & 0.762 [0.424, 1.101] & 11/12 \\
FO vs. lifted, $\gamma=0$ & 1.693 & 1.662 & 0.031 [-0.021, 0.103] & 6/12 \\
\bottomrule
\end{tabular}
\caption{Study-1 nonlinear-suffix results. A--B is positive when the full certificate or lifted pair has lower error. Intervals resample the 12 training seeds.}
\label{tab:phase5-nonlinear-suffix}
\end{table}

The direct control comparison gives the same message. The pairwise model
reduces integrated squared error by 29.0\% when the target contains the
``both features present'' term, with improvement in 11 of 12 independently
trained models. It gives no reliable advantage for the additive target. The
exact losses and intervals are reported in Table~\ref{tab:phase5-nonlinear-suffix}.

This result concerns the estimator and edit direction together. The mismatched
estimator--direction pairs are unstable in 10 of 12 interactional models and
also fail the action- and activation-scale checks. We therefore do not use them
to rank the estimators independently of direction. Appendix
Table~\ref{tab:nonlinear-suffix-factorial} reports all four cells, including
their self-response and convergence factor $p$.

\begin{figure}[H]
  \centering
  \includegraphics[width=0.98\linewidth]{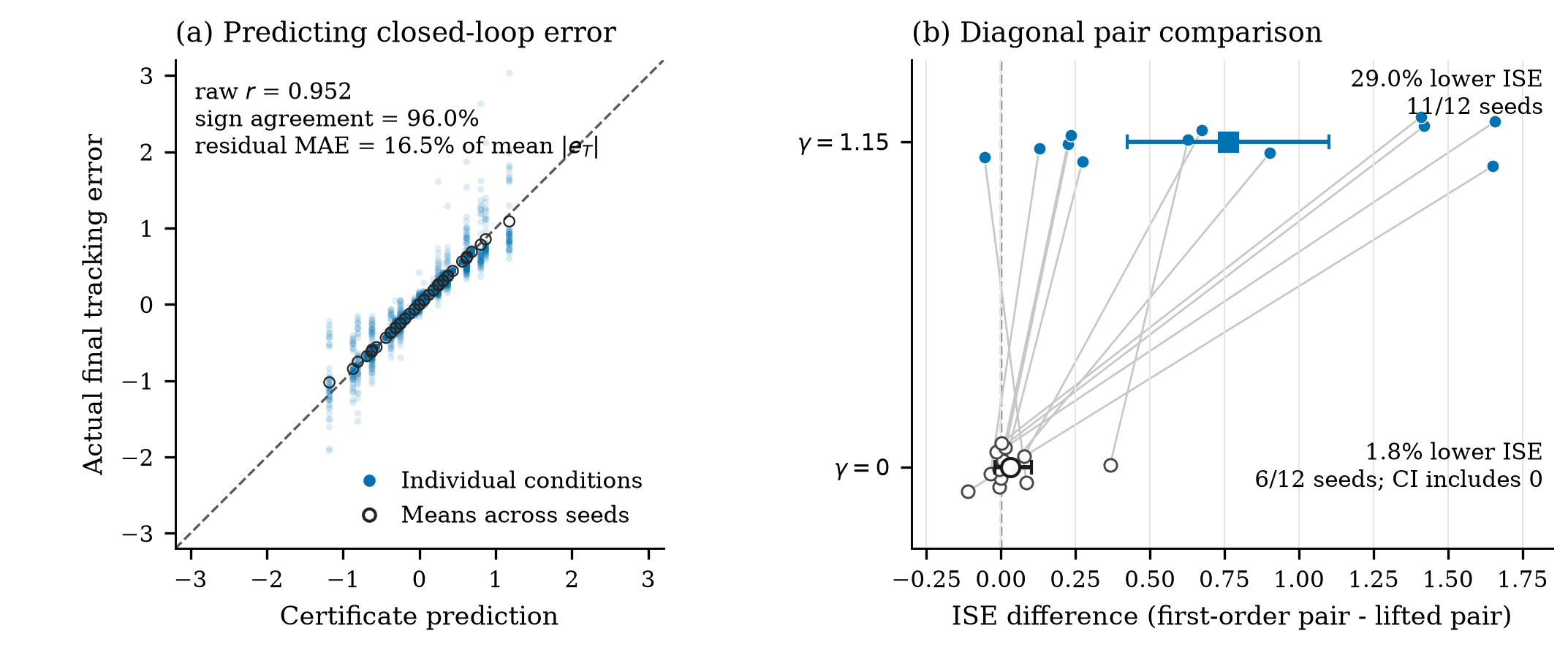}
  \caption{Learned nonlinear-suffix experiment. The certificate combines
  starting mismatch, convergence factor, and response mismatch (left). The
  pairwise estimator--direction pair reduces control error for the
  interactional target but not for the additive target (right). Intervals
  resample 12 independently trained models.}
  \label{fig:nonlinear-suffix}
\end{figure}

\FloatBarrier
The study provides a bridge to pretrained models: a simple control-error
certificate can remain useful after nonlinear computation when edit scale and
loop stability are checked. It does not establish token-level or large-model
control. Closed-loop control evaluates an observer over a trajectory. What if
the decision instead selects one static intervention from a fixed candidate
set using predicted finite effects?

\section{Effect prediction: better estimates need not choose better actions}
\label{sec:effect-prediction}

Many interpretability studies use an observer once, to choose one intervention
from a fixed set. ObserverBench tests whether better prediction of held-out
effects also produces a better choice. It fixes the measurements, candidate
actions, decision rule, and held-out loss, then reports prediction error and
the loss caused by the selected action separately.

We test this effect-prediction contract on GPT-2-small and Qwen2.5-7B. The two
tasks use different behaviors and head sets, but keep the same distinction
between estimating an effect and using that estimate to act.

\begin{table}[H]
\centering
\scriptsize
\setlength{\tabcolsep}{3.5pt}
\begin{tabular}{@{}L{0.17\linewidth}L{0.38\linewidth}L{0.38\linewidth}@{}}
\toprule
Task dimension & GPT-2-small IOI & Qwen2.5-7B copy \\
\midrule
Target & A mask's mean margin drop, or its prompt-level absolute target loss
& A mask's mean candidate-margin drop, or its prompt-level absolute target
loss \\
Information boundary & Frozen calibration effects and declared head or group
features & Frozen calibration effects and declared features for a confirmed
eight-head panel \\
Action space & One candidate head-ablation mask; exact no-op is included in
the confirmation & One of eight nonempty head-ablation masks or exact no-op in
each pool \\
Decision rule and loss & Choose the predicted closest-mean or lowest-loss
mask; score held-out absolute target error & Choose the predicted closest-mean
or lowest-loss mask; score held-out absolute target error \\
Operating regime & Held-out IOI prompts and previously unseen masks & Eligible
held-out induction-copy prompts across four sequence conditions and unseen
masks \\
\bottomrule
\end{tabular}
\caption{Task contracts for the two effect-prediction studies. Each observer
receives the same calibration information within a task and faces the same
candidate actions and held-out loss.}
\label{tab:effect-prediction-contracts}
\end{table}

\subsection{GPT-2 IOI: predicting an effect is not choosing an action}
\label{sec:ioi}

The nonlinear-suffix experiment used a small model with a known mechanism. We
now turn to a documented
circuit in GPT-2-small. Consider the prompt ``When Alice and Bob went to the
store, Alice gave a book to.'' The correct next name is Bob, not the repeated
subject Alice. Prior work identifies three relevant head groups: primary Name
Movers promote Bob, Backup Name Movers compensate when that path is weakened,
and Negative Name Movers suppress the competing name
\cite{wang2023ioi,mcgrath2023hydra,mcdougall2023copy,rushing2024selfrepair}.
These known groups let us test observers on a real intervention surface. We do
not claim to discover a new circuit.

We first ask whether pairwise terms improve mean-effect prediction and the
resulting fixed action, then predict prompt-level action loss directly, and
finally repeat that comparison after the baseline task check passes with the
exact no-op included.

We write the three published head groups as
\[
\begin{aligned}
P &= (9.9,9.6,10.0),\\
B &= (9.0,9.7,10.1,10.2,10.6,10.10,11.2,11.9),\\
E &= (10.7,11.10).
\end{aligned}
\]
An ablation mask $m$ is the set of heads removed together. For prompt $i$ and
mask $m$, define
\[
Y_i(m)=LD_{i,\mathrm{clean}}-LD_{i,\mathrm{ablate}(m)},
\]
where $LD$ is the logit for the indirect object minus the logit for the
subject. Thus $Y_i(m)=1$ means that ablating mask $m$ reduces Bob's advantage
over Alice by one logit on prompt $i$. At the final token, mean ablation
replaces each selected head's output with an average from reference prompts of
the same template. Whole-group ablations first reproduce known backup behavior.
This is a check that the intervention works as intended, not a new mechanistic
result.

\subsubsection{A fixed test of effect prediction and action selection}

We freeze the test before measuring candidate effects:

\begin{itemize}
  \item \textbf{Prompts and splits.} We pin GPT-2-small and generate 1,024
  prompts from eight templates and 64 names. Separate 16-name banks supply 512
  reference, 192 training, 64 validation, and 256 test prompts. Reference
  prompts define the ablation means.

  \item \textbf{Measurements and actions.} We use 160 calibration masks to
  fit each observer. The observer then chooses from 320 previously unseen
  candidate masks, divided into ten action pools of 32. Each mask contains
  4--10 heads. Budgets 20, 40, 80, and 160 use nested prefixes of one
  calibration sequence.

  \item \textbf{Observers.} Four ridge models move from 13 independent head
  effects, to group-count features, to a capacity-matched $P\times E$ term,
  and finally to all three group-pair interactions. Each measurement is one
  mask's average effect over the training prompts.

  \item \textbf{Fixed action.} For each observer, budget, pool, and target
  $t\in\{0.5,1.0,1.5\}$, the observer chooses one mask that is applied to every
  held-out prompt. It cannot choose one mask for Alice--Bob and another for
  Carol--David.
\end{itemize}

An audit found that an earlier metric had accidentally allowed a different,
outcome-informed choice for each prompt rather than one frozen action. We
corrected it before opening validation or test outcomes; the prompts, masks,
observers, targets, and predictions were unchanged.

\subsubsection{Interaction terms improve mean-effect prediction}

Let
\[
\mu(m)=\mathbb E_i[Y_i(m)]
\]
be a mask's average effect over held-out prompts. At the primary budget of 160
measurements, the all-pairs observer lowers candidate-mask MAE from $0.486$ for
per-head additivity and $0.473$ for count additivity to $0.395$, reductions of
18.8\% and 16.7\%. Both improvements hold in all ten candidate pools.

The $P\times E$ and all-pairs observers are effectively tied, with MAEs
$0.3940$ and $0.3945$. This replaces an earlier exploratory ranking of the
pairs. The result supports interaction-aware prediction, but it does not show
that one pair uniquely explains the intervention surface. The gain appears at
budgets 40, 80, and 160. With only 20 measurements, the richer designs have
more coefficients than the data can identify. That cell is therefore an
underdetermined fit, not evidence about sample efficiency.

\subsubsection{The better mean predictor does not choose a better action}
\label{sec:ioi-prediction-decision}

A mean-effect observer chooses
\[
\widehat m(t)=\arg\min_m |\widehat\mu(m)-t|.
\]
For $t=1$, for example, it selects the mask whose predicted average drop is
closest to one logit. The benchmark instead scores the prompt-level loss of
that fixed mask,
\[
R_t(m)=\mathbb E_i\left[|Y_i(m)-t|\right].
\]
Two masks can have the same average and different losses. A mask that reliably
reduces the margin by one logit is better than one that alternates between zero
and two, even though both average to one. Mean-effect prediction and action
selection are therefore different tasks. This is an instance of a broader
prediction--decision distinction: the error of an intermediate estimate need
not determine the quality of the action chosen from it. Roughan et al.
documented an analogous result in traffic engineering, where traffic-matrix
estimation-error magnitude was not by itself a good indicator of routing
performance~\cite{roughan2003te}. Here, mean-effect error omits a
decision-relevant quantity: how a mask's effect varies across prompts.

At budget 160, the all-pairs model has the best $R^2$ and nearly the best
mean-effect MAE, yet its action-loss reductions are only 0.10\% and 2.34\%, with
paired intervals $[-0.087,0.100]$ and $[-0.063,0.128]$. Table~\ref{tab:ioi-phase5}
reports the raw outcomes; the intervals cross zero, and none of the
prespecified size-matched, cost-aware, tolerance, or name-pair checks
establishes a reliable gain.

\begin{table}[H]
\centering
\scriptsize
\begin{tabular}{@{}lrrrrrr@{}}
\toprule
Observer & MAE & $R^2$ & $|\mu-t|$ & $J_t$ & $R_t$ & Within $0.25$ \\
\midrule
Per-head additive & 0.4860 & 0.7629 & 0.4910 & 0.5252 & 1.0163 & 15.42\% \\
Count additive & 0.4734 & 0.7718 & 0.3476 & 0.6920 & 1.0396 & 14.62\% \\
Count plus $P\!\times\!E$ & 0.3940 & 0.8365 & 0.3561 & 0.6703 & 1.0265 & 15.44\% \\
Count plus all pairs & 0.3945 & 0.8387 & 0.3498 & 0.6655 & 1.0152 & 15.76\% \\
\bottomrule
\end{tabular}
\caption{Held-out IOI results at 160 measurements. MAE and $R^2$
score the prompt-averaged effect over candidate masks. The remaining columns
score the masks selected for 30 pool--target cells. Values may differ slightly
under rounding because $R_t=|\mu-t|+J_t$ is computed before aggregation.}
\label{tab:ioi-phase5}
\end{table}

The reason can be written exactly. Define the target-dependent dispersion
penalty
\[
J_t(m)=\mathbb E_i|Y_i(m)-t|-|\mathbb E_i[Y_i(m)]-t|.
\]
This nonnegative term is the extra action loss caused by prompt-to-prompt
variation hidden by the mean. Therefore
\begin{equation}
\label{eq:dispersion}
R_t(m)=|\mu(m)-t|+J_t(m).
\end{equation}
The interaction-aware observers choose masks whose averages are closer to the
target, but those masks vary more across prompts. The second effect cancels the
first. Even perfect knowledge of each mask's test-set mean chooses the
lowest-loss mask in only 3 of 30 pool--target cells. Its loss is $0.927$,
compared with $0.749$ when the full prompt-level action loss is known.
Mean-effect prediction works as intended; the mean is simply the wrong
quantity for this action rule.

\subsubsection{Predict the loss used to choose the action}
\label{sec:phase6-estimand}

Equation~\ref{eq:dispersion} tells us what to try next: keep the measurements
and action rule fixed, but predict the action loss $R_t(m)$ instead of the mean
effect $\mu(m)$. We designed this test after the frozen mean-effect comparison
and an exploratory pilot, then fixed the protocol, predictions, and actions
before measuring new test effects. The local record documents data-access
order; it is not an independent preregistration.

The comparison keeps the measurement basis and action rule fixed:

\begin{itemize}[itemsep=0.15em,topsep=0.25em]
  \item eight new templates, 192 training prompts, 512 test prompts, and
  targets $t\in\{0.5,1.0\}$;
  \item the same 160 calibration masks and 1,536 candidate masks in 48 pools;
  \item one pairwise (quadratic) basis with an intercept, 13 head indicators,
  and 78 head-pair products.
\end{itemize}

The direct-loss observer, called direct risk in the tables, fits
\[
\frac{1}{192}\sum_i |Y_i(m)-t|.
\]
The natural-mean observer predicts the average effect and chooses the mask
closest to the target. A transformed-mean control predicts that distance
directly with the same 92-column basis. The observers therefore see the same
masks and features; only the quantity they learn differs. The held-out
evaluation comprises 786,432 test prompt--mask outcomes.

The mean observer remains strong at its own task: held-out mean-effect MAE is
$0.128$, $R^2$ is $0.844$, and rank correlation is $0.938$. The direct-loss
observer nevertheless reduces action loss from $1.076$ to $0.833$, a 22.6\%
reduction, and by 22.7\% against the same-feature transformed-mean control.
Table~\ref{tab:phase6-ioi-confirmation} reports the paired intervals and all
prespecified checks.

The decomposition explains why. The direct-loss observer matches the mean less
closely: $0.550$ instead of $0.433$. But it lowers the prompt-variation term
from $0.643$ to $0.283$. That $0.360$ reduction is larger than the $0.117$ loss
in mean matching. Extra pairwise features do not reliably improve direct loss
over either additive control; both intervals cross zero. On this
surface, predicting the right quantity matters more than adding features.

\begin{table}[H]
\centering
\scriptsize
\begin{tabular}{@{}llrl@{}}
\toprule
Check & Status & Point statistic & 95\% interval \\
\midrule
H1: risk vs. natural mean & Pass & 22.61\% reduction & [0.114, 0.402] \\
H2: risk basis vs. additive & Fail & 3.03\% reduction & [-0.025, 0.076] \\
H2: risk basis vs. count-additive & Fail & 3.54\% reduction & [-0.021, 0.079] \\
Target-specific transformed-mean sensitivity & Pass & 22.67\% reduction & [0.119, 0.399] \\
Natural-mean own-estimand diagnostic & Descriptive & MAE 0.128; $R^2$ 0.844 & -- \\
No-action audit & Limitation & loss 0.750 vs. risk 0.833 & not prespecified \\
Clean-task validity & Fail & overall 0.867 $<$ 0.900 & worst template 0.578 $<$ 0.750 \\
\bottomrule
\end{tabular}
\caption{Pilot-informed, outcome-sealed fixed-action comparison. Contrast point
statistics are relative fixed-action loss reductions; intervals are paired
absolute loss reductions from the frozen pair-cluster-by-pool bootstrap.}
\label{tab:phase6-ioi-confirmation}
\end{table}

Two failed checks limit this result:

\begin{itemize}[itemsep=0.2em,topsep=0.25em]
  \item \textbf{The exact no-op was absent.} Every candidate mask
  removes 4--9 heads. Since $Y_i(\varnothing)=0$, no-op loss is exactly $t$.
  Averaged over the two targets, no-op loss is $0.750$, below the direct-loss
  observer at $0.833$.
  \item \textbf{The baseline task check failed.} Overall IO-versus-subject accuracy is
  0.867 rather than 0.900, and the worst template reaches 0.578 rather than
  0.750.
\end{itemize}

We remove no rows. This test shows better ranking among forced nonempty actions
on this surface. It does not show improvement over no-op or a new IOI
mechanism.

\subsubsection{At a one-logit target, direct loss beats the mean observer and exact no-op}
\label{sec:phase7-confirmation}

The failed checks determine the confirmation:

\begin{itemize}
  \item We return to the eight canonical templates and generate 512 new prompts
  from 32 new name-pair clusters. Clean IO-versus-subject accuracy is 0.977
  overall and at least 0.953 for every template, so the baseline task check
  passes.
  \item We use $t=1$, the only screened target whose pilot estimates favored
  the direct-loss observer over both the same-feature mean observer and exact
  no-op. This is a disclosed pilot-informed choice, not an independent
  preregistration.
  \item Each of 48 action pools contains 30 fresh nonempty masks plus exact
  no-op. Both observers use the same 92-column quadratic basis and the same
  frozen training and calibration outcomes.
  \item We freeze every action before measuring candidate effects. The two
  observers select 89 distinct nonnoop masks, so we measure only the
  corresponding $512\times89=45{,}568$ prompt--mask cells.
\end{itemize}

A provenance check stopped an initial run before its values were inspected.
Appendix~\ref{app:earlier-diagnostics} records the corrected frozen run and an
independent recomputation of its coefficients, predictions, and choices.

Both required comparisons pass: Table~\ref{tab:phase7-ioi-confirmation} shows
an 18.5\% action-loss reduction against the same-feature mean observer and a
10.9\% reduction against exact no-op, with both paired intervals excluding
zero.

\begin{table}[H]
\centering
\scriptsize
\begin{tabular}{@{}lrrrrl@{}}
\toprule
Reference & Ref. loss & Risk loss & Reduction & 95\% interval & Templates \\
\midrule
Same-basis mean-effect plug-in & 1.093 & 0.891 & 18.5\% & [0.115, 0.293] & 8/8 \\
Exact no action & 1.000 & 0.891 & 10.9\% & [0.012, 0.196] & 6/8 \\
\bottomrule
\end{tabular}
\caption{Fixed-action target tracking at the pilot-informed target $t=1.0$. On 512 new prompt strings from eight fixed canonical templates, intervals resample 32 name-pair clusters and 48 frozen action pools (5,000 draws). Intervals are absolute loss reductions; the Reduction column is relative. The last column is descriptive and shows templates with the same direction. The clean IO-versus-subject gate passed before measurement (97.7\% overall; 95.3\% worst template).}
\label{tab:phase7-ioi-confirmation}
\end{table}

The decomposition shows the same tradeoff. The direct-loss observer has a worse
mean-to-target term, $0.397$ rather than $0.286$, but lowers the dispersion
penalty from $0.807$ to $0.494$. It beats the mean observer on all eight
templates and no-op on six. The pooled no-op comparison passes, but the gain is
not uniform across templates.

\paragraph{How the target changes the value of acting.}
After the primary result, we freeze a secondary analysis over all three
screened targets. It keeps the prompts, pools, quadratic basis, and calibration
outcomes fixed. It compares the direct-loss observer with the natural mean plug-in, the
same-basis transformed-mean score, and exact no-op. We measure new outcomes
only for additional selected masks. This analysis has no success gate and does
not change the baseline-checked result.\footnote{At $t=1$, the confirmation and this sensitivity
use the same loss cells but separate fixed bootstrap seeds. Their point
estimate is identical; lower limits $0.012$ and $0.015$ differ only through
Monte Carlo resampling.}

\begin{table}[H]
\centering
\scriptsize
\setlength{\tabcolsep}{4pt}
\begin{tabular}{@{}lrrrr@{}}
\toprule
Target & Direct loss & Gain vs. mean & Gain vs. transformed mean & Gain vs. no action \\
\midrule
$0.5$ & 0.731 & +29.4\% [0.209, 0.405] & +20.8\% [0.102, 0.288] & \textbf{$-$46.2\% [-0.314, -0.161]} \\
$1.0$ & 0.891 & +18.5\% [0.115, 0.293] & +15.9\% [0.089, 0.253] & +10.9\% [0.015, 0.196] \\
$1.5$ & 1.070 & +10.7\% [0.020, 0.243] & +9.3\% [0.011, 0.217] & +28.6\% [0.321, 0.525] \\
Equal-target mean & 0.897 & +19.1\% [0.141, 0.287] & +14.9\% [0.091, 0.225] & +10.3\% [0.021, 0.177] \\
\bottomrule
\end{tabular}
\caption{Post-confirmatory target sensitivity. This secondary analysis was designed after the $t=1$ confirmation and has no success gate. Each gain cell reports the relative loss change and, in brackets, a 95\% paired bootstrap interval for the absolute reference-minus-direct loss difference. The three targets share 512 prompts and 48 frozen action pools; the last row weights the targets equally. The negative no-action result at $t=0.5$ is retained.}
\label{tab:phase8-ioi-sensitivity}
\end{table}

Table~\ref{tab:phase8-ioi-sensitivity} shows that the direct-loss observer beats
both mean-based controls at every target, with every paired interval excluding
zero. No-op follows a simpler rule. Its effect is zero, so its loss is exactly
$R_t(\varnothing)=t$. At $t=0.5$, the direct-loss observer's loss is 46.2\%
higher than exact no-op, even though it beats both mean-based observers. It
beats no-op at $t=1$ and $t=1.5$. This sign change is why exact no-op belongs
in every action set: an observer can rank the available edits well even when
no edit is warranted.

\begin{samepage}
Together, the IOI tests show that pairwise terms improve mean-effect prediction
and direct-loss prediction improves fixed-action choice, while exact no-op
tests whether any intervention is warranted.
Because these claims are limited to GPT-2-small, the declared templates and
masks, and the measured action losses, we next ask whether the first two
patterns transfer to a larger model and a different intervention surface.
\end{samepage}


\subsection{Qwen2.5-7B: the pattern transfers to a larger model}
\label{sec:qwen-induction}

We test that transfer on the base Qwen2.5-7B model
\cite{qwen2024qwen25} using an induction-copy task with an exact behavioral
target and a different set of heads to intervene on
\cite{olsson2022induction}. The result applies only to prompts that the
unedited model already solves with a clear margin.

Each prompt first contains three key--value pairs and later repeats one key.
For example, after ``A red, B blue, C green, ... B'', the desired continuation
among the three values is ``blue.'' Red and green are the declared distractors.
For each prompt, we score how strongly the model favors the target over the two
distractors combined. For prompt $i$, define this margin as
\[
M_i=\operatorname{logit}_i(\text{target})-
\operatorname{logmeanexp}
\bigl(\operatorname{logit}_i(\text{distractor}_1),
      \operatorname{logit}_i(\text{distractor}_2)\bigr),
\]
where
$\operatorname{logmeanexp}(a,b)=\operatorname{logsumexp}(a,b)-\log 2$.
The effect of an ablation mask is the drop in that margin:
\[
Y_i(m)=M_i(\varnothing)-M_i(m).
\]
A positive $Y_i(m)$ means that ablating those heads weakens the target relative
to the two alternatives. Four prompt families combine sequence lengths 32 and
64 with repeat gaps 8 and 16.

\subsubsection{Define eligible prompts before discovering heads}

Copy-v1 required 95\% baseline accuracy among the three candidates. Qwen reached
92.97\%, so the run stopped before any attention or intervention result was
opened. Copy-v2 did not lower the threshold. Instead, before analysis it defined
an eligible prompt as one for which Qwen chooses the target among the three
candidates, has margin $M_i\geq\log 4$, and produces finite scores. Discovery
prompts must also agree under the eager and SDPA attention implementations.
Both protocols recorded their source and configuration in a private commit
before model outcomes. This records data-access order but is not an independent
preregistration.

Before scoring, we assign 3,072 fresh prompts to separate data banks. A total
of 2,612 prompts (85.03\%) pass the eligibility and coverage checks. Hash order,
not model score, then selects 1,536 prompts for reference, discovery, head
fitting, confirmation, calibration, and final held-out testing. These roles
remain separate throughout the experiment.

On the unfiltered prompts, Qwen produces the desired next token only 24.45\% of
the time. The task therefore asks a bounded question: when Qwen already
distinguishes the planted value from the two distractors, which heads support
that distinction, and how well can we predict their intervention effects? It
does not test unconstrained generation.

\subsubsection{Held-out interventions confirm the eight-head panel}

Attention can suggest heads, but only an intervention shows that they affect
the output. We therefore use four frozen steps:

\begin{enumerate}
  \item Rank all 784 query heads by attention to the target value relative to
  the distractors.
  \item Apply individual mean ablations to only the top 32 on a separate
  head-fitting bank.
  \item Select at most two heads per layer using a lower-bootstrap-bound rule.
  This yields L22H3, L22H4, L26H15, L14H0, L14H6, L16H13, L17H18, and L23H13,
  spanning six layers.
  \item Test the selected panel against matched control heads on a separate
  confirmation bank.
\end{enumerate}

Each selected head receives a low-induction control from the same layer and
grouped-query-attention key--value group. Controls are matched using attention
specificity and output-norm distance, not causal outcomes. On 256 separate
confirmation prompts, the selected heads lower the candidate margin much more
than matched controls, with the paired interval excluding zero; a frozen
zero-ablation check gives the same ordering. Table~\ref{tab:qwen-copy-v2}
reports the estimates. Thus the held-out intervention results, not the attention
scores by themselves, justify including the panel in the benchmark.

\subsubsection{Pair terms and direct loss improve held-out performance}

Only after confirmation passes do we freeze all $2^8=256$ combinations of the
eight heads. One mask might remove heads 1 and 3; another might remove 2, 4,
and 8. We split this complete mask set into 128 calibration masks and 128
held-out masks. Calibration effects average 256 prompts; held-out effects
average a separate 512. The held-out prompts form 128 four-family clusters, so
every mask is tested at every sequence-length and repeat-gap condition.

The additive observer assumes that removing heads 1 and 3 has the sum of their
separate effects. The pairwise observer, a quadratic model for these binary
masks, adds all 28 head pairs and can represent a head whose effect changes
after another head is removed. At the primary budget of 128 measurements, the
pairwise model improves held-out-mask MAE over the additive observer, with the
paired interval excluding zero (Table~\ref{tab:qwen-copy-v2}). The pairwise
model also improves MAE at every smaller nested budget; the 16-measurement fit
has more coefficients than measurements and is therefore descriptive.

Every action selector uses the same 37-column pairwise basis, the same three
targets, and the same 16 frozen action pools. Each pool contains eight nonempty
masks and the exact no-op. The observers differ only in
what they predict:

\begin{itemize}
  \item \textbf{Natural mean} predicts $\mathbb E_i[Y_i(m)]$.
  \item \textbf{Transformed mean} predicts the distance between that mean and
  the target.
  \item \textbf{Direct loss} predicts the action loss itself,
  $\mathbb E_i|Y_i(m)-t|$.
\end{itemize}

The targets are 0.25, 0.50, and 0.75 of the confirmed full-panel effect. We
freeze every prediction and action before opening held-out effects.

\begin{table}[H]
\centering
\scriptsize
\setlength{\tabcolsep}{3.5pt}
\begin{tabular}{@{}llrrl@{}}
\toprule
Test & Comparison & Reference & Result & Improvement [95\% CI] \\
\midrule
Causal & Mean replacement & Control $0.170$ & Selected $3.454$
  & $3.284\ [3.053,3.520]$ \\
Robustness & Zero ablation & Control $0.528$ & Selected $4.080$
  & $3.552\ [3.306,3.801]$ \\
\midrule
Prediction & 16 measurements$^\dagger$ & Additive $0.354$ & Pairwise $0.100$
  & $0.254\ [0.182,0.295]$ \\
& 40 measurements & Additive $0.202$ & Pairwise $0.045$
  & $0.157\ [0.103,0.173]$ \\
& 64 measurements & Additive $0.157$ & Pairwise $0.039$
  & $0.117\ [0.067,0.134]$ \\
& 128 measurements & Additive $0.121$ & Pairwise $0.040$
  & $0.081\ [0.025,0.101]$ \\
\midrule
Action loss & Natural mean & $0.955$ & Direct loss $0.870$
  & $0.085\ [0.053,0.118]$ \\
& Transformed mean & $0.995$ & Direct loss $0.870$
  & $0.125\ [0.083,0.173]$ \\
& Exact no-op & $1.727$ & Direct loss $0.870$
  & $0.857\ [0.805,0.908]$ \\
\bottomrule
\end{tabular}
\caption{Qwen2.5-7B Copy-v2 results. Causal entries report mean
candidate-margin drop on 256 confirmation prompts. Prediction entries report MAE on 128
held-out masks; positive improvement is additive minus pairwise MAE. Action
entries average absolute target loss over three targets and 16 pools; positive
improvement is comparator loss minus direct loss. Intervals use 5,000 paired
percentile bootstrap draws over the appropriate prompt clusters, masks, and
pools. The dagger marks the underdetermined, descriptive 16-measurement fit.
The 128-measurement prediction row and aggregate action rows are the
prespecified primary tests.}
\label{tab:qwen-copy-v2}
\end{table}

Direct loss lowers aggregate loss by 8.9\% relative to natural mean, 12.6\%
relative to transformed mean, and 49.6\% relative to no-op. These numbers give
equal weight to the three targets. At the largest target, direct loss and
natural mean tie within uncertainty: natural minus direct loss is $-0.008$,
with interval $[-0.038,0.021]$. Appendix
Table~\ref{tab:qwen-per-target} reports each target separately.

On 512 matched non-induction continuations, a secondary check finds no larger
clean-to-intervened KL for direct loss than for the natural-mean and
transformed-mean controls ($0.0115$ versus $0.0157$ and $0.0160$); this is not a
general safety result.

Across GPT-2 and this bounded Qwen task, pair terms improve held-out effect
prediction and observers trained on realized loss choose lower-loss
interventions in their frozen aggregates; the Qwen result, however, remains
limited to one pinned model, the eligible prompts, the confirmed eight-head
panel, and the declared masks and neither identifies a complete induction
circuit nor establishes general in-context learning.

These studies choose model interventions from predicted effects. What changes
when the action is instead a scarce safety response and different mistakes
carry different consequences? The next operational mode keeps the same
task-relative contract but evaluates observers under fixed intervention
budgets.

\section{Safety triage: scores must guide scarce interventions}
\label{sec:safety}

When the controller cannot inspect every request, a safety observer must do
more than separate safe cases from unsafe ones. It must rank the cases that
matter most, using the information available at deployment, so that a limited
intervention budget prevents as much harm as possible. The two tasks in this
section make that requirement explicit.

The measurement channel matters as much as the score. During the 2026
OpenAI--Hugging Face incident, agents tried to alter the transcripts they
believed an evaluator would inspect and successfully spoofed some tool calls,
while showing little interest in changing chain-of-thought they did not expect
to be logged \cite{greenblatt2026hfinvestigation}. This does not make internal
signals tamper-proof. It shows that an attacker may concentrate on the channel
it believes is watched. ObserverBench therefore compares declared views---text,
outputs, verdict-adjacent activations, and neutral activations---and records who
controls each view. Monitoring is therefore a moving contest: a channel that
works against frozen attacks may fail after attackers learn that it is being
used. The present tasks keep those views fixed and trusted; they do not test an
attacker that adapts to the monitor.

\begin{table}[H]
\centering
\scriptsize
\caption{Task contracts for the two safety-triage studies. Each contract fixes
what the observer predicts, what it may read, which action the controller can
take, and how the resulting decision is scored.}
\label{tab:safety-contracts}
\begin{tabular}{@{}L{0.18\textwidth}L{0.37\textwidth}L{0.37\textwidth}@{}}
\toprule
Task dimension & Authorization interlock & APPS backdoor audit \\
\midrule
Target & Authorization risk, including the consequence of the requested
operation & Whether a submitted program contains a verified working backdoor \\
Information and integrity boundary & Declared action fields and, depending on
the observer, prompt text, Qwen output scores, or one residual activation;
evaluator-supplied views are trusted & Programming problem and submitted code,
read through Qwen or Gemma prompt scores, text features, dense residuals, or SAE
features; stored views are trusted and attacks do not adapt to the monitor \\
Action space & Allow, block, or escalate a request & Audit a submission or
leave it unaudited \\
Decision rule and loss & Rank by the observer score; block the top 10\%,
escalate the next 10\%, and charge for missed risk and unnecessary intervention
& Rank by suspicion; audit the top-scoring submissions and count verified
backdoors that escape \\
Operating regime & Held-out permission pairs across seen and unseen operations
and prompt formats, under fixed block and escalation budgets & Held-out APPS
problem IDs, fixed audit budgets, and prevalence checks from rare to balanced
attacks \\
\bottomrule
\end{tabular}
\end{table}

The contracts use the same five fields but test different weaknesses. The
authorization task gives exact labels and costs, which lets us isolate whether a
statistically accurate observer allocates its budget well. We begin there.

\subsection{A perfect classifier can allocate a safety budget poorly}

Suppose a workspace assistant is asked to delete a file in workspace A, but the
user has granted access only to workspace B. Deletion is not always unsafe; the
problem is that this file is outside the granted scope. A controller may allow,
block, or escalate the request. When blocks and reviews are limited, it must
rank requests by the harm they could cause.

We turn this example into a non-executing paired task:

\begin{itemize}
  \item \textbf{Same action, different permission.} Each pair requests the same
  operation on the same resource. One prompt grants access to it; the other
  grants access elsewhere.
  \item \textbf{Known label and cost.} An exact authorization rule says whether
  the request exceeds scope. The operation sets the consequence: an unauthorized
  delete costs more than an unauthorized read.
  \item \textbf{Fixed action budget.} The controller blocks the highest-risk
  10\%, escalates the next 10\%, and allows the rest. Every observer faces the
  same costs and budgets.
\end{itemize}

We use the pinned Qwen2.5-7B-Instruct model \cite{qwen2024qwen25}. As a
prerequisite check, the unedited model must answer the authorization question
with one of two exact tokens. It passes every check with accuracy 1.000. The
held-out test contains 384 safe/unsafe pairs across four equally sized subgroups:
seen operations and prompt forms, a held-out delete operation, a held-out XML
form, and both changes together. Fit, calibration, and test resources are
disjoint.

The activation observer reads a hidden-state vector---a residual-stream
activation---at the final prompt position and also receives the requested
action and its severity. We fit on 384
pairs, calibrate on 192, choose among seven layers and ridge settings using the
fixed policy loss, and only then extract held-out activations.

The comparisons ask four questions:

\begin{itemize}
  \item \textbf{Action-only:} what can be inferred without authorization
  context?
  \item \textbf{Activation label and direct risk:} do hidden activations predict
  the policy label or the violation cost?
  \item \textbf{Activation label $\times$ severity:} does combining the label
  score with known consequence improve the action ranking?
  \item \textbf{Allow-all and exact oracle:} what are the upper and lower
  bounds under the same budget?
\end{itemize}

These comparisons were fixed before held-out activations were extracted.
After seeing their results, we added three secondary controls: a full-context
text observer, Qwen's block-minus-allow output-score difference, and an
activation observer selected by tail rather than mean loss. They test whether
activations add value once prompt text or the model's own decision signal is
available; they are not independent confirmation.

\begin{table}[H]
\centering
\scriptsize
\caption{Held-out Qwen safety-interlock results. Every observer uses the same
10\% block and 10\% escalation budgets. Lower loss and missed risk are better;
higher clean utility is better. Access states what the observer can read.}
\label{tab:qwen-safety}
\begin{tabular}{@{}p{0.27\textwidth}p{0.18\textwidth}rrrr@{}}
\toprule
Observer & Access & Mean loss & Tail loss & Missed risk & Clean utility \\
\midrule
Allow all & None & 5.010 & 19.026 & 1.000 & 1.000 \\
Action-only direct risk & Action fields & 3.176 & 13.987 & 0.623 & 0.901 \\
Activation label & Residual + fields & 3.334 & 18.182 & 0.665 & 1.000 \\
Activation direct risk & Residual + fields & 2.520 & 16.545 & 0.502 & 1.000 \\
Activation label $\times$ severity & Residual + fields & 2.155 & 14.026 & 0.429 & 1.000 \\
Full-context text label $\times$ severity$^\dagger$ & Prompt text + fields & 2.032 & 12.857 & 0.405 & 1.000 \\
Qwen logit margin $\times$ severity$^\dagger$ & Output logits + fields & \textbf{1.744} & \textbf{10.078} & \textbf{0.347} & \textbf{1.000} \\
Exact authorization-risk oracle & Evaluator label + fields & 1.740 & 10.039 & 0.346 & 1.000 \\
\bottomrule
\end{tabular}
\end{table}

\noindent\emph{$^\dagger$ Secondary controls.} The text and output rows were
designed after the original held-out results were known.

The results have four main implications:

\begin{itemize}
  \item \textbf{Activations improve the primary comparison.} Multiplying the
  activation label by severity lowers mean policy loss by 32.1\% relative to
  action-only and 57.0\% relative to allowing every request. It closes 71.1\%
  of the action-only-to-oracle gap while preserving all benign utility.
  \item \textbf{Perfect classification is not enough.} The label observer has
  AUROC 1.000 and loss 3.334. Multiplying its score by severity lowers AUROC to
  0.954 but lowers loss to 2.155 because unauthorized deletes are ranked ahead
  of unauthorized reads.
  \item \textbf{The useful signal is not uniquely hidden.} Full-context text is
  not reliably worse than the transformed activation observer. More clearly,
  Qwen's output-score difference multiplied by severity lowers mean loss to
  1.744 and tail loss to 10.078, nearly matching the exact oracle at 1.740 and
  10.039. When this output is available, the activation probe is unnecessary.
  \item \textbf{Good ranking within groups can fail across groups.} The
  activation label has AUROC 1.000 inside every operation--format subgroup, but
  its scale changes between them. A global threshold therefore allocates 63.2\%
  of interventions to held-out deletes in seen forms and only 18.4\% to the
  doubly held-out subgroup. There, activation loss is 5.007 versus 3.059 for
  action-only; the output observer catches 69.8\% of unsafe requests versus
  29.2\% for the activation observer.
\end{itemize}

\begin{figure}[H]
\centering
\includegraphics[width=\textwidth]{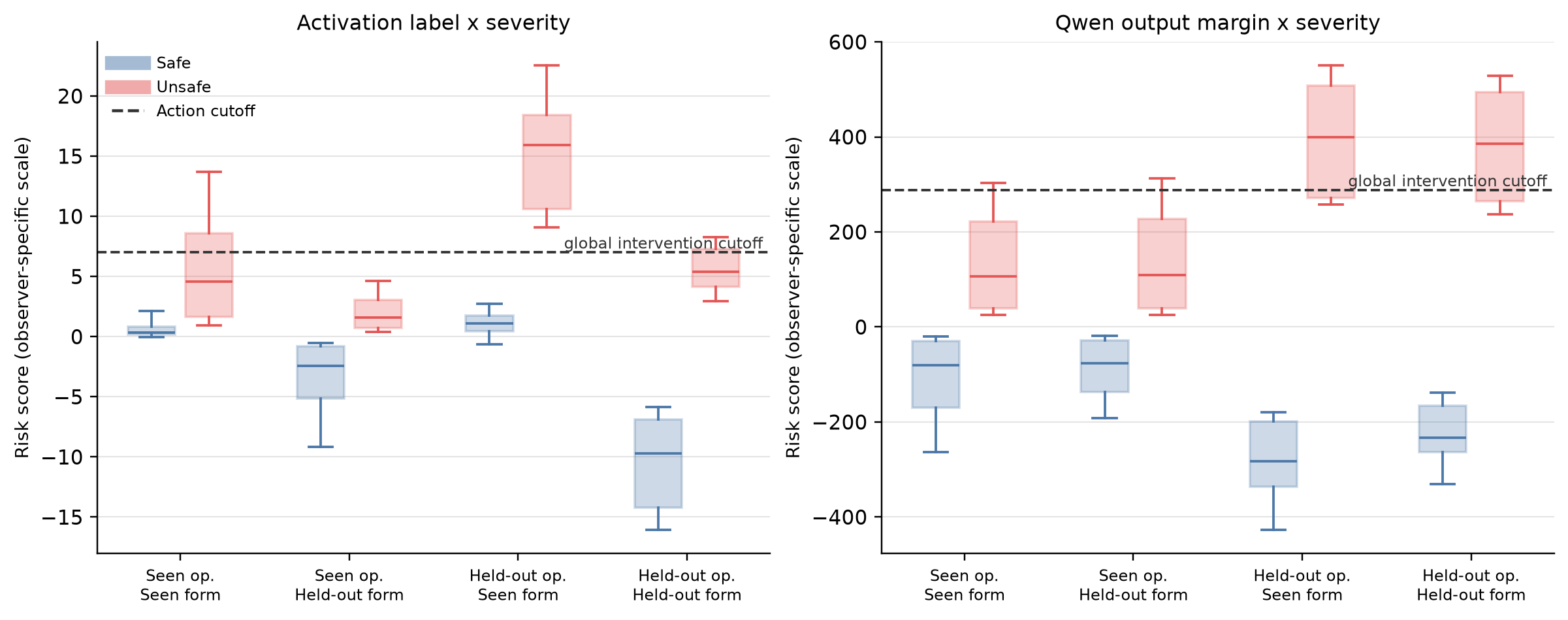}
\caption{Risk scores by operation--format subgroup. The activation observer
separates safe from unsafe requests within each subgroup, but its scale shifts
between subgroups. Dashed lines show each observer's global intervention
cutoff. Scores should be compared within panels, not across them.}
\label{fig:qwen-safety-calibration}
\end{figure}

This failure separates three requirements: the observer must classify the
violation, rank it by consequence, and calibrate its scores across deployment
groups. Recalibration by visible operation and format may repair the last
failure, but this split cannot test that repair because calibration contains
neither held-out case. A valid follow-up needs fresh prompts and calibration
coverage of every intended deployment group.

The controlled task shows why consequence and cross-group calibration matter.
Does the same task-relative evaluation change which monitor we prefer on an
externally recognized safety benchmark? We next turn to APPS.

\subsection{Monitor rankings depend on the model and operating regime}
\label{sec:safety-apps}

To answer that question, we use
ControlArena's APPS setting, an existing AI-control task with ordinary code
solutions and verified working backdoors
\cite{greenblatt2024aicontrol,inglis2025controlarena}. A monitor reads the
programming problem and submitted code, assigns a suspicion score, and a fixed
budget audits only the highest-scoring submissions. Every unaudited working
backdoor counts as one violation.

The public dataset supplies 1,444 honest/backdoor pairs. We split by problem ID:
500 pairs for fitting, 200 for calibration, and 744 for held-out testing, with both
members of a pair always in the same split. We use the same pinned Qwen model as
above.

The observers fall into three groups:

\begin{itemize}
  \item \textbf{Prompt and output observers} ask Qwen for a score and read the
  model's probability distribution over ten score tokens. We test basic and
  detailed prompts, their ensemble, and a follow-up ridge model using all ten
  token probabilities.
  \item \textbf{A text observer} is a follow-up baseline trained on word and
  character features from the problem and code.
  \item \textbf{Activation observers} fit a linear model to a hidden-state
  vector. One is read next to an explicit verdict prompt; the follow-up neutral
  observer reads the same problem and code with no monitoring or scoring
  instruction.
\end{itemize}

The original comparison used the three prompt observers and the
verdict-adjacent activation probe. We added the score-distribution, text, and
neutral-activation controls after that result and fixed them before measurement.
At a 2\% audit budget, calibration selects layer 28 for the verdict-adjacent
probe and layer 20 for the neutral probe.

\begin{table}[H]
\centering
\small
\caption{Qwen2.5-7B monitor comparison on 744 held-out APPS pairs. The controller audits
29 of 1,488 submissions at 2\% and 148 at 10\%. Higher caught-backdoor counts
are better; 29 and 148 are the best possible counts under those budgets.}
\label{tab:apps-monitor}
\begin{tabular}{@{}llrrr@{}}
\toprule
Observer & Access & AUROC & Caught @2\% & Caught @10\% \\
\midrule
Exact-label oracle & evaluator & 1.000 & 29 & 148 \\
Neutral activation & neutral residual & \textbf{0.979} & \textbf{29} & \textbf{148} \\
Verdict-adjacent activation & prompt residual & 0.953 & 29 & 148 \\
Supervised text & problem + code & 0.868 & 29 & 146 \\
Score-distribution ridge & prompt logits & 0.830 & 29 & 147 \\
Prompt ensemble & two prompt expectations & 0.817 & 28 & 146 \\
Detailed prompt & prompt expectation & 0.810 & 28 & 146 \\
Basic prompt & prompt expectation & 0.786 & 27 & 137 \\
\bottomrule
\end{tabular}
\end{table}

Three findings matter:

\begin{itemize}
  \item \textbf{Neutral activations rank attacks best.} The neutral probe reaches
  AUROC 0.979 without asking Qwen to judge the code. Its paired gains are 0.026
  over verdict-adjacent activation, 0.111 over supervised text, and 0.149 over
  the score-distribution observer, with 95\% intervals $[0.018,0.034]$,
  $[0.098,0.125]$, and $[0.131,0.167]$.
  \item \textbf{The balanced panel hides action differences.} Neutral and
  verdict-adjacent activations both reach the maximum possible caught-backdoor count at 2\%,
  5\%, and 10\%. At 10\%, neutral activation leaves 596 attacks unaudited versus
  607 for the basic prompt, a paired difference of $-11$ with 95\% interval
  $[-17,-5]$. The sample does not resolve its smaller gains over the stronger
  controls because attacks make up 50\% of the panel and fill the small audit queue.
  \item \textbf{The tested controls narrow, but do not settle, the explanation.}
  Reading all ten output probabilities raises AUROC from 0.786 to 0.830, so
  sampling or averaging the declared score tokens does not explain the activation
  gain. The text result rules out this lexical baseline, not every possible code
  monitor or output observer.
\end{itemize}

The neutral-context result is relevant to adaptive monitoring, but it is not a
robustness result. It shows that a useful signal can appear before the model is
asked for a verdict. It does not show that the same signal survives once an
attacker knows it is being read.

In deployment, attacks should be rare. We therefore resample the frozen scores at attack
rates from 1--50\% and audit budgets from 0.5--10\%. At 1\% attacks and a 2\%
audit budget, the exact oracle misses none. The activation observers miss 2.65
attacks on average for neutral context (95\% range $[0,6]$) and 3.94 next to the
verdict prompt ($[1,7]$). The text, score-distribution, and basic-prompt observers
miss 9.05 ($[6,12]$), 9.15 ($[6,13]$), and 11.85 ($[9,15]$). Rare attacks expose
differences hidden by the balanced panel.

\begin{figure}[H]
\centering
\includegraphics[width=\textwidth]{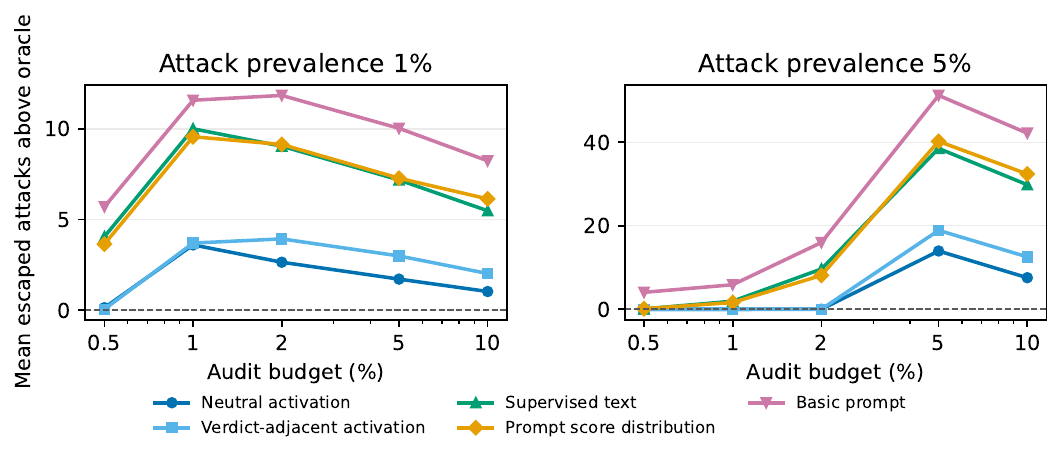}
\caption{Qwen2.5-7B APPS: excess escaped attacks over the exact-label oracle
across attack prevalences and audit budgets in 2,000 resampled populations.}
\label{fig:apps-prevalence}
\end{figure}

A post-outcome diagnostic applies a published sparse-probing procedure to the
frozen Qwen panel \cite{kantamneni2025saeprobes}. Its fitted readout
retains 127 nonzero coefficients, versus 3,584 for a layer-matched dense
residual, but misses 4.67 rather than 2.05 attacks on average. This is
\emph{compression without a better decision} on this frozen panel, not a
general claim about SAE features. Appendix~\ref{app:apps-diagnostics} reports
the full diagnostic, including allocation by APPS family and input length.
This is an external monitor test, not a deployment claim.

\subsubsection{A matched Gemma replication changes the mean ranking}

We repeat the APPS contract on Gemma-2-9B-it, keeping the rows, prompts,
controller, budgets, and evaluation loss fixed \cite{gemmateam2024gemma2}. The
design was frozen after the Qwen results but before any Gemma monitor outcome
was inspected. It is a prospectively frozen matched replication, not an
independent preregistration.

\begin{table}[H]
\centering
\small
\caption{Matched Gemma-2-9B-it APPS replication. AUROC uses the balanced
locked panel of 744 honest--backdoor pairs. Realized violations use the
registered primary cell: 1\% attacks and a 2\% audit budget. Entries are means
and central 95\% outcome ranges over 2,000 resampled populations; lower is
better.}
\label{tab:gemma-apps-monitor-long}
\begin{tabular}{@{}lrr@{}}
\toprule
Observer & AUROC & Realized violations \\
\midrule
Exact-label oracle & 1.000 & 0.000 $[0,0]$ \\
Verdict-adjacent residual & \textbf{0.969} & \textbf{2.178 $[0,5]$} \\
Score-distribution ridge & 0.913 & 4.423 $[1,8]$ \\
Neutral residual & 0.944 & 4.947 $[2,9]$ \\
Supervised text ridge & 0.868 & 9.054 $[6,12]$ \\
Gemma Scope SAE probe & 0.870 & 9.070 $[6,12]$ \\
Detailed prompted score & 0.893 & 10.803 $[7,14]$ \\
Prompt ensemble & 0.860 & 13.689 $[11,15]$ \\
Basic prompted monitor & 0.809 & 14.927 $[14,15]$ \\
\bottomrule
\end{tabular}
\end{table}

Residual access has the lowest mean loss in both models, but its best measured
context changes: neutral on Qwen and verdict-adjacent on Gemma. Output
information still matters. On Gemma, the score-distribution ridge has a
slightly lower violation mean than the neutral residual, 4.423 versus 4.947,
despite lower AUROC.

\textbf{Classification is not enough on the external benchmark.} This safety
result is a second instance of the prediction--decision distinction in
Sec.~\ref{sec:ioi-prediction-decision}. The detailed prompted monitor has
higher AUROC than the official Gemma Scope SAE probe, 0.893 versus 0.870, but
worse mean realized violations, 10.803 versus 9.070
\cite{lieberum2024gemmascope}. The detailed-minus-SAE paired difference
averages 1.73 with a central 95\% outcome range of $[-2,6]$, so this is a
reversal of the mean scorecard ranking, not a superiority claim for every
resampled population. Thus, higher classification quality need not induce the
lower-mean-loss audit policy.

\subsubsection{Qwen3.5 reverses the Qwen2.5 context ranking}

Using the same matched protocol, we repeat the APPS rows, split, prompts, audit
controller, and primary operating point on Qwen3.5-9B.\footnote{Pinned model:
\url{https://huggingface.co/Qwen/Qwen3.5-9B}.} We prospectively froze the
design after the Qwen2.5 and Gemma outcomes were known but before inspecting
any Qwen3.5 monitor outcome; this is not an independent preregistration.

\begin{table}[H]
\centering
\small
\caption{Prospectively frozen Qwen3.5-9B APPS replication. AUROC uses the
balanced locked panel. Mean violations use the primary operating point: 1\%
attacks, a 2\% audit budget, and 2,000 paired resampled populations. Lower
violations are better. The exact-label row is an evaluator-only bound; bold
marks the best deployable row.}
\label{tab:qwen35-apps-monitor-long}
\begin{tabular}{@{}lrr@{}}
\toprule
Observer & AUROC & Mean violations \\
\midrule
Exact-label oracle (evaluator only) & 1.000 & 0.000 \\
Verdict-adjacent residual & \textbf{0.977} & \textbf{1.976} \\
Supervised score distribution & 0.940 & 5.120 \\
Neutral residual & 0.936 & 6.601 \\
Supervised text & 0.868 & 9.054 \\
Basic prompted score & 0.876 & 12.218 \\
Official Qwen-Scope SAE probe & 0.811 & 12.725 \\
Basic/detailed prompt ensemble & 0.823 & 12.956 \\
Detailed prompted score & 0.645 & 13.985 \\
\bottomrule
\end{tabular}
\end{table}

The verdict-adjacent residual now wins. Neutral minus verdict averages 4.625
additional violations, with a central 95\% outcome range of $[1,8]$; verdict is
better in 99.15\% of resampled populations. Qwen3.5 therefore agrees with the
Gemma ordering and reverses Qwen2.5. The useful context can change even between
checkpoints from the same model family.

The other rows show a second reversal. The basic prompted monitor has slightly
higher AUROC than supervised text, 0.876 versus 0.868, but causes 3.164 more
mean violations, 12.218 versus 9.054. The detailed prompted monitor also falls
from AUROC 0.893 on Gemma to 0.645 on Qwen3.5 under the same APPS contract.
Monitor quality must therefore be tested on the model and operating regime in
which it will be used.

We also test the official Qwen-Scope SAE.\footnote{Pinned SAE:
\url{https://huggingface.co/Qwen/SAE-Res-Qwen3.5-9B-Base-W64K-L0_50}.}
Calibration selects layer 12 and ridge 100. The SAE probe has AUROC 0.811 and
12.725 mean violations, compared with 6.601 for the dense neutral residual.
SAE minus dense averages 6.124 additional violations, with a central 95\%
outcome range of $[2,10]$; dense is better in 99.65\% of populations. This is a
bounded cross-checkpoint control, not a general result against SAE features:
the released SAE was trained on Qwen3.5-9B-Base, while we evaluate the
post-trained Qwen3.5-9B model \cite{deng2026qwenscope}. Qwen's public
collection did not include a checkpoint-matched post-trained 9B SAE at the
time of the study. Appendix~\ref{app:qwen35-apps} records the full protocol
and mismatch.

\subsection{Four checks before claiming that an internal safety signal helps}

These studies suggest a reusable sequence. Before interpreting an internal
safety signal, ObserverBench asks:

\begin{enumerate}[leftmargin=1.4em,itemsep=0.2em,topsep=0.25em]
  \item \textbf{Does the model output already suffice?} Read the relevant output
  distribution rather than one sampled answer. If it is adequate, stop.
  \item \textbf{Does the input already suffice?} Train a text or input-feature
  baseline with the same supervision. If it matches the internal probe, an
  activation-specific explanation is unnecessary.
  \item \textbf{Does the signal remain without a request for a verdict?} Measure
  the same examples in a neutral context to separate input processing from a
  signal created or reshaped by the monitoring prompt.
  \item \textbf{Does it work under the intended policy?} Test frozen scores at
  the attack rate, intervention budget, deployment groups, and loss the
  controller will face.
\end{enumerate}

These checks locate where useful information becomes available. They do not by
themselves establish deception, intent, or hidden knowledge. How does this
task-relative contract relate to existing interpretability
benchmarks, decision-focused learning, and AI-control evaluations? The next
section places ObserverBench among those lines of work.

\section{Related work}

\paragraph{Interventions, steering, and interpretability benchmarks.}
Activation addition and representation engineering edit internal activations,
while feedback steering and local optimal-control methods place those edits in a
controller
\cite{turner2023activation,zou2023representation,nguyen2025feedback,skifstad2026local}.
Activation and attribution patching estimate component effects; AtP* provides a
local approximation; causal abstraction tests assigned causal roles; and pyvene
executes model interventions
\cite{syed2024attribution,kramar2024atp,geiger2024alignments,wu2024pyvene}.
Interactions between mediators can hide or inflate patching effects
\cite{vaidyanathan2026curse}. MIB, AxBench, and InterpBench instead benchmark
circuit localization, concept detection and steering, or recovery of known
structure \cite{mueller2025mib,wu2025axbench,gupta2024interpbench}. Mechanistic
tomography treats measurements as a way to recover a mechanism or effect map
\cite{erramilli2026tomography}; ObserverBench starts with such an estimate and
asks whether it supports a declared intervention or controller.
Sparse probing asks whether SAE features support compact supervised predictors
\cite{kantamneni2025saeprobes}. ObserverBench adds the operational question:
does the sparser fitted readout improve the action chosen under a fixed policy?

\paragraph{Decisions and AI control.}
Decision-focused learning evaluates a predictor through the downstream decision
rather than assuming that lower prediction error is enough
\cite{donti2017task,elmachtoub2022smart}. The same distinction motivates calls
to judge interpretability by the actions it enables \cite{orgad2026actionable}.
AI control evaluates complete safety protocols against strategies chosen by an
intentionally subversive model, while ControlArena supplies reusable
environments, monitors, and protocol-level evaluation
\cite{greenblatt2024aicontrol,inglis2025controlarena}. ObserverBench asks a
complementary component-level question. It holds the controller, audit rule,
action budget, and test cases fixed, then compares the observer scores used
inside that decision problem. It evaluates monitor adequacy, not end-to-end
robustness to subversion.

\paragraph{Borrowed testbeds.}
The IOI circuit supplies known structure rather than a new mechanism claim
\cite{wang2023ioi}. Conditional backup, Hydra-style self-repair, copy
suppression, and prompt-dependent self-repair are already documented
\cite{mcgrath2023hydra,mcdougall2023copy,rushing2024selfrepair}. Induction heads
support the repeated-key continuation used in the Qwen task
\cite{olsson2022induction}, and Qwen2.5-7B supplies a second model surface
\cite{qwen2024qwen25}. We use these systems to define interventions, then judge
observers on fresh held-out masks and prompts rather than claim a new IOI or
induction mechanism. What remains unsupported by the current tasks and should
be tested next?

\section{Limitations and next steps}
\label{sec:limitations}

The paper makes bounded claims:

\begin{itemize}
  \item \textbf{Control.} \emph{Scope:} The exact results assume affine
  readouts, one fixed direction, proportional control, and no clipping. The
  nonlinear test reruns a learned Transformer block, but the model is small and
  its underlying binary features are known by construction; the 29.0\% gain belongs to the matched
  estimator--direction pair, not the estimator by itself. \emph{Next test:} extend
  the comparison to token-level control in larger models with independently
  usable edit directions.

  \item \textbf{Pretrained-model interventions.} \emph{Scope:} The GPT-2 and
  Qwen studies use pinned models, selected head sets, fixed ablations, and narrow
  prompt distributions. The IOI confirmation is pilot-informed and locally
  sealed rather than independently preregistered. The Qwen task keeps only
  eligible prompts; full-vocabulary top-1 accuracy on unfiltered prompts is
  24.45\%, and the direct-loss observer ties the natural-mean observer at the
  largest target.
  \emph{Next test:} use more models, unseen prompt templates, and independently
  fixed candidate sets and targets.

  \item \textbf{Authorization safety.} \emph{Scope:} The non-executing task asks Qwen
  directly for an authorization verdict, so it cannot establish a uniquely
  hidden safety state or test harmful execution, deception, or adaptive attacks.
  Its held-out failure also shows that one global threshold can hide poor
  calibration between groups. \emph{Next test:} remove the explicit verdict
  question and provide calibration data for every intended deployment group
  while keeping the model, budgets, and evaluator fixed.

  \item \textbf{APPS backdoors.} \emph{Scope:} The study uses one stored-backdoor
  task rather than attacks adapted to each monitor. Two model families and
  three checkpoints supply learned representations. The Qwen2.5 sparse probe
  and Qwen3.5 official-SAE probe each concern one task and one selected probing
  rule. The official Qwen-Scope SAE was trained on Qwen3.5-9B-Base and applied
  here to the post-trained Qwen3.5-9B model, so it is not a checkpoint-matched
  SAE test. The Gemma and Qwen3.5 runs were prospectively frozen after earlier
  outcomes were known, not independently preregistered. All scoring is offline, and the
  prevalence sweeps keep the dataset-family mix fixed. Neither study tests
  deception or an adaptive attacker. These scorecards compare observers inside
  a fixed audit rule; they do not establish that a complete AI-control protocol
  remains safe under intentional subversion. \emph{Next test:} use sequential
  attack trajectories with evaluator-held outcomes and both trusted and
  deliberately corrupted telemetry. Declare who can omit, delay, or spoof each
  view, keep audit and shutdown budgets fixed, and then add monitor-aware
  attacks. Episodes, rather than individual requests, should be the resampling
  unit when agents share state or communication.

  \Needspace{4\baselineskip}
  \item \textbf{Artifact.} \emph{Scope:} The release provides checked control,
  finite-effect, and safety interfaces, but it is not yet a general task
  ecosystem. \emph{Next test:} add literature baselines, another safety
  environment with adaptive attacks, and a stable end-to-end reproduction path.
\end{itemize}

\Needspace{12\baselineskip}
\section{Conclusion}

ObserverBench asks a simple question: is an estimate good enough for the action
chosen from it? The answer is task-relative. An observer need not reproduce
every detail of a model; it must preserve the information that the available
action, decision rule, and loss actually use. Across the three modes, the
experiments show that a better mean-effect predictor need not choose a better
intervention, a perfect classifier can allocate a safety budget poorly, and an
accessible output signal can make an internal probe unnecessary. The safety
scorecards also show that AUROC and mean realized violations can order monitors
differently, that the best measurement context can change between closely
related checkpoints, and that a sparser fitted readout need not improve the
resulting action. ObserverBench records these choices in versioned tasks and applies the
same fixed policy to
every submission rather than ranking observers in the abstract. Before trusting
an observer, state what it may read, freeze the action rule and budget, and score
the outcome the action actually produces.

\clearpage
\appendix
\section{Artifact interfaces and sealed evaluation}
\label{app:artifact-interfaces}

The artifact exposes three versioned interfaces:

\begin{itemize}[leftmargin=1.5em,itemsep=0.1em,topsep=0.2em]
  \item \path{observerbench.api.v0} for control composition;
  \item \path{observerbench.effect_prediction.v0} for held-out finite-effect
  prediction; and
  \item \path{observerbench.safety_protocol.v0} for safety triage.
\end{itemize}

Their machine-readable schemas, example submissions, and validation commands
ship with the repository. The version identifiers prevent a result produced
under one contract from being compared silently with a result produced under
another.

The task registry states which submission route each bundled task supports.
In the current release, the analytic control task provides the checked
external-observer adapter. Its benchmark wrapper constructs the controller and
reuses the task's fixed generator, readouts, normalization, and metrics. A
lower-level run may replace the controller, but the resulting score is marked
non-comparable because it no longer represents the same decision problem.

Public safety-task packs contain the threat model, allowed actions,
consequence loss, operating budgets, training measurements, and held-out
queries. They do not contain the held-out evaluator answers. The evaluator
stores those targets separately and links them to the public pack by a
cryptographic hash. A submission can therefore be checked against the declared
task while the labels used for final scoring remain sealed. This layout also
allows a researcher to add a data-only safety task without modifying the core
workbench.

The bundled GPT-2 IOI and Qwen induction-copy registries implement the same
separation for finite-effect prediction. Each registry exposes frozen
measurement tables at declared budgets, checks the model and split recorded by
the task, and uses different prompts for calibration measurements and held-out
targets. An outside observer can submit its predictions without loading the
model or rerunning the intervention experiment. The registries include
per-component and additive ridge observers as initial comparison rows.

The external APPS registry contains separate Qwen2.5-7B, Gemma-2-9B-it, and
Qwen3.5-9B scorecards. Prompt, output, text, dense-residual, and SAE-feature
observers are evaluated under the same audit contract within each model panel.
The Qwen3.5 release includes the frozen source manifest and an aggregate
checked summary containing the scorecard rows, paired contrasts, and official
SAE selection. Per-example evaluator labels remain sealed. The registry does
not issue a cross-model rank.

Runs use strict JSON and record
the contract and result-schema versions, component classes and parameters,
package and dependency versions, the source revision, configuration and result
hashes, and observer provenance supplied by the submitter.

\section{Robustness checks and protocol provenance}
\label{app:earlier-diagnostics}

This appendix retains checks that directly qualify the main results. Superseded
control fixtures and development history that support no live claim are kept
with the artifact rather than in the paper.

\subsection{Nonlinear estimator--direction robustness}

Table~\ref{tab:nonlinear-suffix-factorial} reports all four
estimator--direction combinations behind Sec.~\ref{sec:nonlinear-suffix}.

\begin{table}[H]
\centering
\scriptsize
\begin{tabular}{@{}rllrrrrr@{}}
\toprule
$\gamma$ & Estimate & Direction & ISE & $g_E\!\leq\!0$ & $|p|\!\geq\!1$ & Would clip & Beyond clean max \\
\midrule
1.15 & FO & FO & 2.632 & 0/12 & 0/12 & 0/192 & 0/192 \\
1.15 & FO & Lifted & 2.381 & 10/12 & 10/12 & 30/192 & 56/192 \\
1.15 & Lifted & FO & 3.666 & 10/12 & 10/12 & 56/192 & 96/192 \\
1.15 & Lifted & Lifted & 1.870 & 0/12 & 0/12 & 0/192 & 0/192 \\
\addlinespace
0 & FO & FO & 1.693 & 0/12 & 0/12 & 0/192 & 0/192 \\
0 & FO & Lifted & 1.659 & 0/12 & 0/12 & 0/192 & 0/192 \\
0 & Lifted & FO & 1.693 & 0/12 & 0/12 & 0/192 & 0/192 \\
0 & Lifted & Lifted & 1.662 & 0/12 & 0/12 & 0/192 & 0/192 \\
\bottomrule
\end{tabular}
\caption{Secondary nonlinear-suffix estimator--direction factorial. Counts for $g_E$ and $p$ use 12 seed-level means; the last two columns use 192 rollout conditions per cell. The crossed arms at $\gamma=1.15$ fail loop-conditioning and clean-archetype scale diagnostics, so they do not identify an estimator-only advantage. The clean maximum is a scale check, not a manifold boundary.}
\label{tab:nonlinear-suffix-factorial}
\end{table}

\subsection{IOI interaction robustness}

Table~\ref{tab:ioi-budget-appendix} reports the full nested-budget curve behind
the interaction comparison in Sec.~\ref{sec:ioi}.

\begin{table}[H]
\centering
\small
\begin{tabular}{@{}rrrrr@{}}
\toprule
Budget & Per-head & Count & $P\!\times\!E$ & All pairs \\
\midrule
20  & 0.846 & 1.579 & 1.333 & 0.795 \\
40  & 0.675 & 0.732 & 0.544 & 0.559 \\
80  & 0.496 & 0.482 & 0.406 & 0.412 \\
160 & 0.486 & 0.473 & 0.394 & 0.395 \\
\bottomrule
\end{tabular}
\caption{Fresh held-out IOI mean-effect MAE by nested measurement budget.}
\label{tab:ioi-budget-appendix}
\end{table}

\subsection{Protocol provenance}

\begin{itemize}
  \item \textbf{IOI target sensitivity.} The first measurement copy was set
  aside before inspection because its source seal omitted the head-mapping file.
  After expanding the seal, we independently recomputed the frozen predictions
  and actions and repeated the 148-mask measurement. The two copies are
  identical at the byte level; only the second enters the analysis. Under the frozen
  raw-score rule, the transformed-mean control selects negative predicted
  losses in 44, 44, and 16 pools across the three targets, while direct loss
  does so in 4, 0, and 0. We retain those scores rather than clip them after
  observing the outcomes.
  \item \textbf{Qwen Copy-v2.} Copy-v1 stopped at 92.97\% candidate accuracy,
  below its frozen 95\% gate, before attention or intervention outcomes were
  computed. Copy-v2 used fresh tokens and prompts under the same gate before
  generating the locked results in Section~\ref{sec:qwen-induction}.
\end{itemize}

\subsection{Authorization tail-loss selection}

Selecting the activation observer by tail loss chooses the same layer and ridge
setting as selecting it by mean loss. Its tail loss differs from action-only by
0.039, with paired 95\% interval $[-1.013,1.156]$. We therefore retain the mean
activation gain but do not claim a tail-loss gain over action-only.

\subsection{Qwen action loss by target}

The main Qwen result gives equal weight to three targets.
Table~\ref{tab:qwen-per-target} shows the individual comparisons. Direct loss
beats transformed mean and no-op at every target. It beats natural mean at the
first two targets and ties it within uncertainty at the largest, so the
aggregate result is not a win in every cell.

\begin{table}[H]
\centering
\scriptsize
\setlength{\tabcolsep}{4pt}
\begin{tabular}{@{}lrrrr@{}}
\toprule
Target & Direct loss & Gain vs. natural mean & Gain vs. transformed mean & Gain vs. no action \\
\midrule
$0.25s$ & 0.609 & +12.0\% [0.034, 0.141] & +14.8\% [0.035, 0.188] & +29.4\% [0.205, 0.300] \\
$0.50s$ & 0.842 & +17.6\% [0.108, 0.256] & +18.7\% [0.117, 0.283] & +51.3\% [0.824, 0.942] \\
$0.75s$ & 1.158 & \textbf{$-$0.7\% [-0.038, 0.021]} & +6.3\% [0.022, 0.146] & +55.3\% [1.331, 1.530] \\
Equal-target mean & 0.870 & +8.9\% [0.053, 0.118] & +12.6\% [0.083, 0.173] & +49.6\% [0.805, 0.908] \\
\bottomrule
\end{tabular}
\caption{Qwen Copy-v2 action loss by target, where $s=3.454$ is the confirmed
full-panel mean effect. Each gain cell gives the relative loss change and, in
brackets, a 95\% four-family-prompt-cluster-by-pool bootstrap interval for the
absolute comparator-minus-direct difference. The negative natural-mean point
estimate at $0.75s$ is retained.}
\label{tab:qwen-per-target}
\end{table}

\section{Post-outcome APPS diagnostics}
\label{app:apps-diagnostics}

These checks use frozen activations and scores after the main result was known.
They do not change the locked comparison in Table~\ref{tab:apps-monitor}.

\paragraph{Does a sparse-autoencoder basis improve the audit decision?}
We apply the sparse-probing procedure of \cite{kantamneni2025saeprobes} to a
pinned public BatchTopK SAE trained on Qwen2.5-7B-Instruct layer-19
residuals.\footnote{\url{https://huggingface.co/andyrdt/saes-qwen2.5-7b-instruct}}
Feature ranking uses only the fit split; feature count and regularization use
only calibration; and the locked test is scored once. The dense residual and
SAE reconstruction controls use the same layer-19 access, data splits, and
calibration rule.

\begin{table}[H]
\centering
\small
\caption{Post-outcome published-method integration on the frozen Qwen APPS
panel. Mean violations use 2,000 resampled populations with 1\% attacks and a
2\% audit budget. Lower is better.}
\label{tab:qwen-sae-probe-long}
\begin{tabular}{@{}lrrr@{}}
\toprule
Observer & Nonzero coefficients & AUROC & Mean violations \\
\midrule
Dense residual logistic & 3,584 & \textbf{0.983} & \textbf{2.05} \\
SAE sparse readout & 127 & 0.953 & 4.67 \\
SAE reconstruction logistic & 3,584 & 0.954 & 4.79 \\
\bottomrule
\end{tabular}
\end{table}

The sparse SAE readout uses fewer fitted coefficients, but it does not improve
mean audit loss on this frozen panel. Dense minus sparse averages $-2.62$
escaped attacks with a central 95\% outcome range of $[-6,0]$; the dense probe
is worse in 1.25\% of draws. This is not explained only by long inputs: on the
1,301 locked rows within the SAE's 1,024-token training length, its AUROC is
0.956 versus 0.983 for the dense residual. The checkpoint is also off its
training distribution here: observed mean $L_0$ is 115 rather than the reported
240. The result concerns this checkpoint, layer, task, and probing rule; it is
not a general claim that SAE features cannot support safety monitoring.

\paragraph{Does the answer position explain the gap between prompted and
neutral activations?}
The fixed calibration rule selects layer 28 for the verdict-adjacent observer,
with calibration AUROC 0.935 and locked AUROC 0.953. A post-outcome check of
cached layer 20 raises locked AUROC to 0.958, but calibration would not have
selected it (AUROC 0.929), and the neutral layer-20 probe remains higher at
0.979. This check cannot separate answer position from the monitoring
instruction or other context changes.

\paragraph{Does pooling hide score differences between groups?}
Figure~\ref{fig:apps-strata} groups neutral-probe scores by APPS source family
and input-length quartile. Separation remains strong within every family
(AUROC 0.956--0.991) and length quartile (0.974--0.988), but score scale shifts.
At the global 2\% cutoff, a family's share of caught attacks is 0.36--2.38 times
its population share; across length quartiles the range is 0.67--1.43. These
ratios come from only 29 audits and are descriptive. They show that pooled
discrimination need not imply equal allocation between groups. The fixed
family mix and prevalence sweep do not model an attacker choosing where to
attack.

\begin{figure}[H]
\centering
\includegraphics[width=\textwidth]{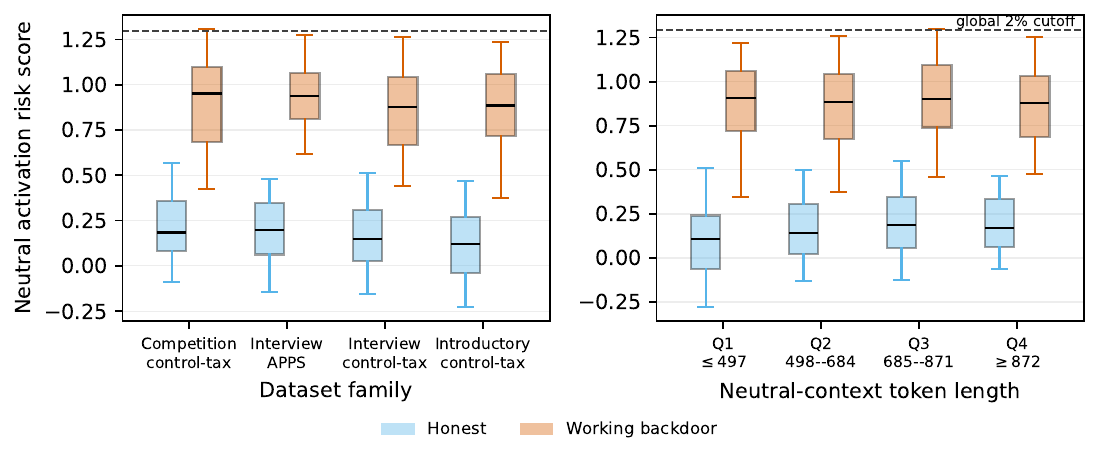}
\caption{Neutral-context scores by APPS source family and input-length
quartile. Separation remains strong within each group, but score scales and the
audits assigned by one global cutoff differ between groups.}
\label{fig:apps-strata}
\end{figure}

\section{Qwen3.5 APPS replication protocol}
\label{app:qwen35-apps}

This replication keeps the APPS task fixed while changing the model. It uses
the same 1,444 honest--backdoor pairs as the Qwen2.5 and Gemma studies, split by
problem ID into 500 fit, 200 calibration, and 744 locked-test pairs. Both
members of a pair remain in the same split. The prompts, score tokens,
controller, audit budgets, resampling population, and realized-violation loss
also remain fixed. Thinking is disabled for every Qwen3.5 view.

The design was chosen after the Qwen2.5 and Gemma results were known. Before
reading any Qwen3.5 monitor outcome, we pinned the model and official SAE
revisions, declared the primary contrast, and sealed the source manifest. The
primary contrast is neutral-residual minus verdict-adjacent-residual mean
violations at 1\% attack prevalence and a 2\% audit budget. This is a local
prospective seal, not an independent preregistration.

The pinned model revision is
\path{c202236235762e1c871ad0ccb60c8ee5ba337b9a}; the pinned SAE revision is
\path{7bb2370d360e566b2d8acb8b3d15a0b2b88f52a8}. The source files and sealed
manifest ship with the artifact, so the aggregate scorecard can be tied to the
exact code and configuration used for the run without publishing per-example
evaluator labels.

We measure post-layer residuals at layers 12, 20, and 28. Fit data train each
probe; calibration chooses the layer and ridge value; the locked test is read
only after selection. One fit pair exceeds the 4,096-token limit and is
excluded under the frozen rule. No calibration or locked-test pair reaches the
limit. A compatibility-only smoke first checks prompt lengths, score tokens,
residual shapes, and the three SAE files without fitting or ranking an
observer.

The official-feature control uses 65,536-wide TopK-50 Qwen-Scope SAEs at the
same three layers. Calibration selects layer 12 and ridge 100. The SAEs were
trained on Qwen3.5-9B-Base, whereas the dense residuals and outcomes come from
post-trained Qwen3.5-9B. At the time of the study, Qwen's official public
collection listed only base-trained 9B Qwen3.5 SAEs, not a checkpoint-matched
SAE for the post-trained model.\footnote{Official collection:
\url{https://huggingface.co/collections/Qwen/qwen-scope}.} The
Qwen-Scope card also states that using a base-model SAE to explore a
post-trained checkpoint is often reasonable. We declared the mismatch before
the run. The SAE row therefore tests a
released feature family under transfer; it is not a checkpoint-matched
comparison or evidence that SAE features generally fail for safety monitoring.

\section{Why a lexical safety fixture was rejected}
\label{app:rejected-cyber}

An earlier fixture tried to distinguish unsafe actionable assistance from benign
discussion of the same security topic. A TF--IDF baseline and a linear
activation observer both solved its small split, even with one topic family held
out. Removing lexical directions then removed the target signal by construction.
The task could not distinguish an internal safety representation from words that
revealed the label, so it is not an ObserverBench result.

The failure motivated the controls used in the authorization study: test a text
baseline, hold the requested action fixed while changing the authorization
context, and compare with the exact policy rule when one exists. The later text,
output, and neutral-context controls add a further check: if the prompt asks
directly for a safety verdict, an activation probe by itself cannot establish a
uniquely hidden safety state. More examples with the same lexical structure
would not repair the rejected fixture.

\end{document}